%% file: main.tex
\documentclass{article}

\usepackage{microtype}
\usepackage{graphicx}
\usepackage{subfigure}
\usepackage{booktabs} 

\usepackage{hyperref}

\usepackage[accepted]{icml2026}

\usepackage{amsmath}
\usepackage{amssymb}
\usepackage{mathtools}
\usepackage{amsthm}
\input{math_commands}

\usepackage[capitalize,noabbrev]{cleveref}

\usepackage{setspace}

\theoremstyle{plain}
\newtheorem{theorem}{Theorem}[section]

\theoremstyle{definition}

\newtheorem{assumption}[theorem]{Assumption}
\theoremstyle{remark}

\usepackage[textsize=tiny]{todonotes}

\usepackage{multirow}
\usepackage{makecell}
\usepackage{enumitem}
\usepackage{makecell}
\usepackage{colortbl}
\usepackage{wrapfig}    

\newcommand{\ul}{\underline}

\icmltitlerunning{DualOptim+}

\begin{document}

\twocolumn[
\icmltitle{DualOptim+: Bridging Shared and Decoupled Optimizer States for \\Better Machine Unlearning in Large Language Models}




\begin{icmlauthorlist}
\icmlauthor{Xuyang Zhong}{cityu}
\icmlauthor{Qizhang Li}{ind}
\icmlauthor{Yiwen Guo}{ind}
\icmlauthor{Chen Liu}{cityu}
\end{icmlauthorlist}

\icmlaffiliation{cityu}{Department of Computer Science, City University of Hong Kong}
\icmlaffiliation{ind}{Independent Researcher}

\icmlcorrespondingauthor{Chen Liu}{chen.liu@cityu.edu.hk}
\icmlcorrespondingauthor{Yiwen Guo}{guoyiwen89@gmail.com}

\icmlkeywords{Machine Learning, ICML}

\vskip 0.3in
]



\printAffiliationsAndNotice{}  

\begin{abstract}
We propose \textbf{DualOptim+}, a novel optimization framework for improving machine unlearning in large language models.
It introduces a base state to capture common representations shared by forgetting and retaining objectives and delta states to preserve objective-specific residuals. This architecture allows the optimizer to adaptively bridge shared and decoupled states based on the directional conflict between forgetting and retaining gradients. We further introduce DualOptim+ 8bit, a quantized variant that reduces memory overhead without compromising performance. Extensive experiments across fictitious and real-world unlearning, safety alignment, and multi-task learning tasks demonstrate that DualOptim+ consistently achieves a superior trade-off between different objectives. Codes are available at \href{https://github.com/CityU-MLO/DualOptimPlus}{https://github.com/CityU-MLO/DualOptimPlus}.
\end{abstract}

\input{Section/intro}

\input{Section/related}
\input{Section/method}
\input{Section/experiments}
\input{Section/conclusion}


\bibliography{main}
\bibliographystyle{icml2026}

\newpage
\appendix
\onecolumn
\input{Appendix/alg}

\input{Appendix/discussion}
\input{Appendix/experiment}
\input{Appendix/imple}

\end{document}

%% file: math_commands.tex
\usepackage{amsmath,amsfonts,bm}

\def\eqref#1{equation~\ref{#1}}

\def\1{\bm{1}}

\def\vtheta{{\bm{\theta}}}

\def\vg{{\bm{g}}}

\def\vm{{\bm{m}}}

\def\vo{{\bm{o}}}

\def\vv{{\bm{v}}}

\DeclareMathAlphabet{\mathsfit}{\encodingdefault}{\sfdefault}{m}{sl}
\SetMathAlphabet{\mathsfit}{bold}{\encodingdefault}{\sfdefault}{bx}{n}

\def\gD{{\mathcal{D}}}

\def\gL{{\mathcal{L}}}

\def\gP{{\mathcal{P}}}

\newcommand{\R}{\mathbb{R}}



%% file: Section/intro.tex
\section{Introduction}
Machine unlearning (MU) \citep{bourtoule2021machine} aims to erase the influence of specific training data, known as the forget set, from pretrained models while preserving their general utility on the retain set. While MU has been extensively studied in computer vision tasks, such as image classification and generation \citep{heng2023selective, kurmanji2023towards, fan2024salun, huang2025unified}, the rapid rise of Large Language Models (LLMs) has underscored the need for efficient methods to remove outdated or unauthorized information \citep{dang2021right}. In this regard, specialized unlearning techniques for LLMs have recently emerged as a significant area of research \citep{yao2024large, zhang2024negative, yuan2024closer}.

Despite recent progress, it is still challenging to balance the erasure of specific information and the preservation of general capability by optimizing various designed unlearning objectives. Most existing LLM unlearning methods \citep{yao2024large, zhang2024negative, yuan2024closer} minimize the objectives of the forget set and the retain set jointly by the sum of their gradients, which often leads to a significant degradation in model utility. Inspired by \citet{fan2024salun, huang2025unified}, alternately optimizing forgetting and retaining objectives shows promising unlearning performance, but it suffers from gradient entanglement for two conflicting objectives when using a shared optimizer state (e.g., the moving average of gradients and squared gradients in Adam). To address this issue, DualOptim \citep{zhong2025dualoptim} decouples optimizer states and uses separate optimizers for different objectives. Despite the effectiveness in computer vision tasks, it yields marginal improvements for LLMs. Therefore, developing a novel updating scheme is essential for effective LLM unlearning.

In this work, we propose \textbf{DualOptim+}, a plug-and-play framework compatible with any optimizer with stored states. It introduces a shared \textbf{base state} alongside decoupled \textbf{delta states} for each optimization objective. Specifically, the base state is updated using gradients from both the forgetting and retaining objectives, enabling it to capture their common representations. In parallel, the delta states are updated by the residual between the objective-specific gradients and the base state, thereby preserving distinct representations unique to each objective. 
Finally, the parameters are updated by combining the base state and the delta state.

Our theoretical and numerical analyses demonstrate that DualOptim+ functions as an adaptive intermediate between fully shared and decoupled states, adjusting its behavior based on the degree of directional conflict between the forgetting and retaining gradients. We validate the effectiveness of DualOptim+ through extensive experiments across diverse machine unlearning tasks on various LLMs, including fictitious datasets, real-world scenarios, and safety alignment tasks. To mitigate the memory overhead associated with additional optimizer states, we also introduce \textbf{DualOptim+ 8bit}. This quantized variant significantly reduces memory consumption, while maintaining the peak performance. Our results indicate that DualOptim+ bridges the gap between decoupled and shared optimizer states to achieve a superior trade-off between forgetting efficacy and model utility. 
We believe, our method is a generalizable optimization framework, which can be applied in broader scenarios beyond machine unlearning,  such as LLM alignment, multi-objective learning, e.t.c.

We summarize the contributions of this paper as follows:
\begin{enumerate}
    \item We propose \textbf{DualOptim+}, which introduces a shared base state to capture common representations and decoupled delta states to preserve task-specific residuals, effectively bridging the gap between shared and decoupled optimizer states.
    \item DualOptim+ is a plug-and-play framework applicable to any multi-objective optimization and optimizers with stored states. Theoretical and numerical analysis demonstrate that DualOptim+ functions as an intermediate between fully shared and decoupled states.
    \item Extensive experiments across various LLMs, datasets and tasks confirm that our method achieves a superior trade-off between forgetting efficacy and model utility. DualOptim+ 8bit reduces memory overhead by quantization without compromising performance.
\end{enumerate}

%% file: Section/related.tex
\section{Related Work}
Early efforts of machine unlearning (MU) are devoted to tasks such as image classification and image generation~\citep{bourtoule2021machine, heng2023selective, jia2023model, kurmanji2023towards, tarun2023fast, fan2024salun, huang2025unified}. It has recently been adapted to address the unique challenges of Large Language Models (LLMs), such as removing sensitive or copyrighted training data \citep{mainitofu}, mitigating harmful behaviors \citep{li2024the}, and efficiently handling the high-dimensional parameter space \citep{fan2024simplicity}. 
Based on how the model handles forgotten knowledge, MU methods on LLMs can be categorized into \textit{untargeted} and \textit{targeted} unlearning.

In \textbf{untargeted unlearning}, the objective is to eliminate the influence of specific data without constraining the model's subsequent response to the forgotten content. Common techniques for this paradigm include gradient ascent (GA) \citep{Thudi2021OnTN, yao2024large}, negative preference optimization (NPO) \citep{fan2024simplicity, zhang2024negative}, maximum entropy (ME) \citep{yuan2024closer}, and representation misalignment \citep{li2024the, zou2024improving}. Conversely, \textbf{targeted unlearning} aims to induce specific model behaviors when encountering forgotten information, such as providing standardized rejection responses (e.g., ``I don’t know''). It is more user-friendly than untargeted unlearning. Popular methods for targeted unlearning include rejection fine-tuning (IDK) \citep{mainitofu}, direct preference optimization (DPO) \citep{rafailov2023direct}, and self-classification \citep{gandikota2024elm}.

Besides erasing targeted information, it is also crucial for both untargeted and targeted unlearning to maintain general model utility and avoid catastrophic forgetting. Most existing approaches incorporate cross-entropy loss \citep{gandikota2024elm, li2024the, yao2024large, zhang2024negative} or divergence-driven loss \citep{yuan2024closer} to optimize the model utility on a designated retain set.

Most LLM unlearning methods jointly update the objectives for the forget and the retain sets, but this optimization strategy usually leads to excessive forgetting and utility degradation \citep{zhang2024negative}.
This issue is mitigated by alternatively using gradients from the forgetting and the retaining objectives \citep{kurmanji2023towards, fan2024salun, huang2025unified}.
DualOptim \citep{zhong2025dualoptim} further improves the effectiveness and stability of unlearning by using two distinct optimizers with separate states.

In addition to the aforementioned optimization-based methods \citep{Jeung2025DUSKDN,jeung-etal-2025-seps,reisizadeh-etal-2026-blur}, inference-time adjustment methods aim to achieve efficient unlearning without modifying model parameters. Theses methods are mainly based on output intervention \citep{deng2025guard, wang2026dragon} and in-context learning \citep{pawelczyk2024incontext}.

%% file: Section/method.tex
\section{Method}
\subsection{Preliminaries}
\begin{figure*}[!t]
    \centering
    \subfigure[Joint]{
    \includegraphics[width=0.125\textwidth]{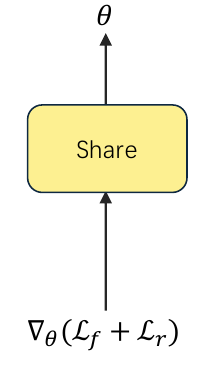}
    }
    \subfigure[Alternate]{
    \includegraphics[width=0.125\textwidth]{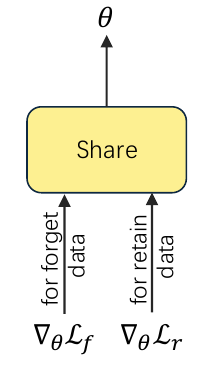}
    }
    \subfigure[DualOptim]{
    \includegraphics[width=0.295\textwidth]{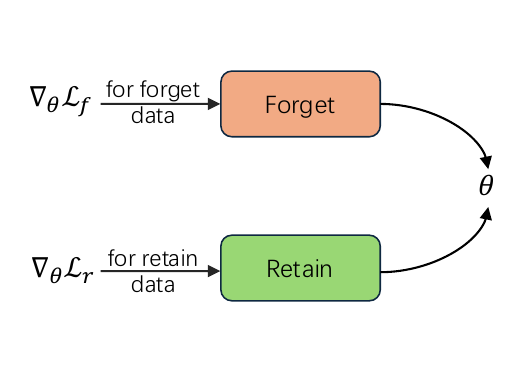}
    }
    \subfigure[\textbf{DualOptim+}]{
    \includegraphics[width=0.365\textwidth]{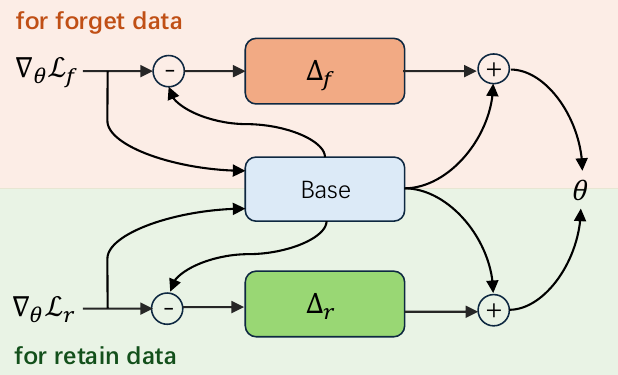}
    } 
    \vspace{-0.5em}
    \caption{\textbf{Comparison of baselines and DualOptim+}. The block represents the optimizer state. \textbf{(a) Joint} updating scheme sums both the forgetting and retaining losses and executes a single back-propagation step. \textbf{(b) Alternate} updating scheme minimizes (\ref{eq:problem}) by gradients from either $\gL_f$ or $\gL_r$ at each iteration, and alternates between the two objectives. \textbf{(c) DualOptim} introduces an independent optimizer state for each objective based on Alternate. \textbf{(d)} \textbf{DualOptim+} bridges the shared and decoupled states by introducing base and delta states. The base state is jointly updated by $\nabla_\theta \gL_f$ and $\nabla_\theta \gL_r$, and $\Delta_f$ / $\Delta_r$ is updated by the difference between $\nabla_\theta \gL_f$ / $\nabla_\theta \gL_r$ and base state. Ultimately, the momentum term used for the parameter update is reconstructed by combining the base state with the respective delta state.} 
    \label{fig:baseline_compare}
    \vspace{-0.5em}
\end{figure*}

Machine unlearning (MU) seeks to solve the following optimization problem:
\begin{equation} \label{eq:problem}
    \min_{\bm{\theta}}\gL_f(\bm{\theta}, \gD_f) + \gL_r(\bm{\theta}, \gD_r),
\end{equation}
where $\bm{\theta}$ represents the model parameter. The forget and retain sets are denoted by $\gD_f$ and $\gD_r$, respectively. $\gL_f$ and $\gL_r$ are the corresponding loss functions for forgetting and retaining objectives, respectively. These objectives guide the model to eliminate the information contained in $\gD_f$ while simultaneously preserving the utility on $\gD_r$.

In the context of large language models (LLMs), most MU methods \citep{zhang2024negative,yuan2024closer} employ the \textbf{Joint} updating scheme (see Figure \ref{fig:baseline_compare} (a)). This scheme sums both the forgetting and retaining objectives and obtains the gradient by a single back-propagation step to minimize (\ref{eq:problem}). The updating step can be formulated as follows. $\gP$ represents an optimizer that may store states.
\begin{equation}
    \vtheta \leftarrow \vtheta - \gP(\nabla_{\vtheta}(\gL_f + \gL_r)).
\end{equation}
Some recent works \citep{fan2024salun,huang2025unified} utilize the \textbf{Alternate} updating scheme (see Figure \ref{fig:baseline_compare} (b)) to improve the performance. This scheme minimizes (\ref{eq:problem}) by gradients from either $\gL_f$ or $\gL_r$ at each iteration, and alternates between the two objectives. Formally,
\begin{equation} \label{eq:update_alter}
\begin{cases}
    \vtheta \leftarrow \vtheta - \gP(\nabla_{\vtheta}\gL_f) & \text{for forget data}\\
    \vtheta \leftarrow \vtheta - \gP(\nabla_{\vtheta}\gL_r) & \text{for retain data}
\end{cases}.
\end{equation}
Despite its effectiveness, the alternating updating scheme makes MU approaches unstable and sensitive to hyper-parameter tuning.
\textbf{DualOptim} \citep{zhong2025dualoptim} (see Figure~\ref{fig:baseline_compare} (c)) mitigates the unstable issue and further improves the effectiveness by using two decoupled optimizers $\gP_f$, $\gP_r$ with separate states.
This updating scheme disentangles the conflicting gradients from two objectives. Formally,
\begin{equation} \label{eq:update_dualoptim}
\begin{cases}
    \vtheta \leftarrow \vtheta - \gP_f(\nabla_{\vtheta}\gL_f) & \text{for forget data}\\
    \vtheta \leftarrow \vtheta - \gP_r(\nabla_{\vtheta}\gL_r) & \text{for retain data}
\end{cases}.
\end{equation}
DualOptim is a plug-and-play technique applicable for various unlearning tasks.
However, compared with remarkable performance improvement in image classification and generation, applying DualOptim in LLMs yields relatively marginal improvement.
Our observations in Figure~\ref{fig:numerical_results} (a) show that the similarity between the decoupled momentum terms of DualOptim is near zero, implying that the gradients from the forgetting and retaining objectives in LLM unlearning are not always conflicting, especially in the later phase.
This phenomenon indicates that we should adaptively utilize both shared features and distinct features from the gradients.


\subsection{Bridging Shared and Decoupled Optimizer States}
In this subsection, we propose \textbf{DualOptim+} to bridge the decoupled and shared optimizer states to improve the unlearning performance for LLMs (see Figure~\ref{fig:baseline_compare} (d)).
Specifically, DualOptim+ decomposes each optimizer state into two components: a shared \textbf{base state} and decoupled \textbf{delta states} for forgetting and retaining objectives. Without loss of generality, we use the first-order momentum term as an example in the analyses below. The updating rule can be straightforwardly extended to other optimizer states.

\textbf{Base State.} The base state $B$ is introduced to reserve the common representation shared by the forgetting and retaining objectives. It is updated jointly by $\nabla_\theta \gL_f$ and $\nabla_\theta \gL_r$, effectively acting as a shared state. Formally,
\begin{equation} \label{eq:base_update}
\begin{cases}
    B \leftarrow \beta B + (1-\beta) \nabla_{\theta} \gL_f & \text{for forget data} \\
    B \leftarrow \beta B + (1-\beta) \nabla_{\theta} \gL_r & \text{for retain data}
\end{cases}.
\end{equation}
where $\beta\in[0,1)$ is the momentum factor. 

\textbf{Delta State.} We introduce the delta states $\Delta_f$, $\Delta_r$ to capture the historical residual information: the difference between the individual gradient and the base state. This allows the delta states to reserve the distinct representations specific to the forgetting and retaining objectives. Formally,
\begin{equation} \label{eq:delta_update}
\begin{cases}
    \Delta_f \leftarrow \beta \Delta_f + (1-\beta) (\nabla_{\theta} \gL_f -\widehat{B}) & \text{for forget data} \\
    \Delta_r \leftarrow \beta \Delta_r + (1-\beta) (\nabla_{\theta} \gL_r -\widehat{B}) & \text{for retain data}
\end{cases}.
\end{equation}
where $\widehat{B} = B / (1-\beta^t)$ denotes the bias-corrected base state and $t$ is the iteration number.

\begin{algorithm}[!t]
   \caption{DualOptim+ with AdamW}
   \label{alg:do+_adamw}
   \setstretch{1.15}
\begin{algorithmic}[1]
    \STATE {\bfseries Input:} parameter $\vtheta$, learning rate $\eta$, betas ($\beta_1$, $\beta_2$), epsilon $\epsilon$, weight decay factor $\lambda$, forget objective $\gL_f$, retain objective $\gL_r$, total steps $N$, forget frequency $F_f$, retain frequency $F_r$
    \STATE {\bfseries Initialize:} $\vm_{\Delta_f}\leftarrow 0$, $\vm_{\Delta_r}\leftarrow 0$, $\vm_B\leftarrow 0$, $\vv_{\Delta_f}\leftarrow 0$, $\vv_{\Delta_r}\leftarrow 0$, $\vv_B\leftarrow 0$, $t_f\leftarrow 0$, $t_r\leftarrow 0$
    \STATE $\widehat{\vm}_B,~\widehat{\vv}_B \leftarrow \vm_B,~\vv_B$
    \FOR{$t=1$ \textbf{to} $N$}
        \IF{$t\bmod (F_f + F_r) \leq F_f$} 
            \STATE $t_f \leftarrow t_f + 1$
            \STATE $\vg,~\vm_{\Delta},~\vv_{\Delta},~t' \leftarrow \nabla_{\vtheta}{\gL_f}(\vtheta),~\vm_{\Delta_f},~\vv_{\Delta_f}, ~t_f$
        \ELSE
            \STATE $t_r \leftarrow t_r + 1$
            \STATE $\vg ,~\vm_{\Delta},~\vv_{\Delta},~t' \leftarrow \nabla_{\vtheta}{\gL_r}(\vtheta),~\vm_{\Delta_r},~\vv_{\Delta_r},~t_r$
        \ENDIF
        \STATE $\vtheta \leftarrow \vtheta - \eta\lambda \vtheta$
        \STATE $\vm_{\Delta} \leftarrow \beta_1\vm_{\Delta} + (1-\beta_1) (\vg - \widehat{\vm}_B)$
        \STATE $\vv_{\Delta} \leftarrow \beta_2\vv_{\Delta} + (1-\beta_2) (\vg^2 - \widehat{\vv}_B)$
        \STATE $\widehat{\vm}_{\Delta},~\widehat{\vv}_{\Delta} \leftarrow \vm_{\Delta} / (1-\beta_1^{t'}),~\vv_{\Delta} / (1-\beta_2^{t'})$
        \STATE $\vtheta \leftarrow \vtheta - \eta (\widehat{\vm}_B + \widehat{\vm}_{\Delta}) / (\sqrt{|\widehat{\vv}_B + \widehat{\vv}_{\Delta}|} + \epsilon)$
        \STATE $\vm_B \leftarrow \beta_1\vm_B + (1-\beta_1) \vg$
        \STATE $\vv_B \leftarrow \beta_2\vv_B + (1-\beta_2) \vg^2$
        \STATE $\widehat{\vm}_B,~ \widehat{\vv}_B \leftarrow \vm_B / (1-\beta_1^{t}),~ \vv_B / (1-\beta_2^{t})$
    \ENDFOR
    \STATE {\bfseries Output:} parameter $\vtheta$
\end{algorithmic}
\end{algorithm}

\textbf{Parameter Update.} Finally, the momentum term used for the parameter update is reconstructed by combining the bias-corrected base state ($\widehat{B}$) with the respective bias-corrected delta state ($\widehat{\Delta}_f$ or $\widehat{\Delta}_r$), i.e., $\widehat{B} + \widehat{\Delta}_f$ or $\widehat{B} + \widehat{\Delta}_r$. For Adam and its variants, the second-order momentum is derived in a similar manner using a separate set of base and delta states updated by the squared gradients $(\nabla_\theta \gL_f)^2$ or $(\nabla_\theta \gL_r)^2$. The full pseudo-code of DualOptim+ integrated with AdamW is presented in Algorithm \ref{alg:do+_adamw} where we alternatively utilise $\nabla_\theta \gL_f$ for $F_f$ iterations and $\nabla_\theta \gL_r$ for $F_r$ iterations. It should be noted that the base state is updated after the parameter update to maintain a stable reference for the delta states, thereby enhancing optimization stability.

It should be emphasized that DualOptim+ is a generic framework that can be integrated into any optimizer with stored states. As an additional example, the pseudo-code for DualOptim+ integrated into Muon \cite{jordan2024muon} is provided in Algorithm \ref{alg:do+_muon} of Appendix \ref{sec:app_do+_muon}.

In addition, our method share the similar idea with some federated learning methods \citep{pmlr-v119-karimireddy20a, karimireddy2021mime, fedcm, local_adam}, i.e., the base + delta decomposition. However, our method focuses on unlearning task, which is distinct from federated learning in terms of objectives. We defer the detailed discussion and comparison in Appendix \ref{sec:compare_fl}.

\subsection{Analyses on DualOptim+} \label{sec:further_discussion}
The design of DualOptim+ yields a hypothesis regarding its behavior: DualOptim+ acts as \textbf{an intermediate between fully shared state in (\ref{eq:update_alter}) and fully decoupled states in (\ref{eq:update_dualoptim})}, adapting based on the degree of directional conflict between $\nabla_\theta \gL_f$ and $\nabla_\theta \gL_r$.

\textbf{Theoretical Analysis.}  
Without the loss of generality, we focus on one coordinate for notation simplicity since DualOptim+ updates the parameters in an elementwise manner.
We consider the assumption below before detailed analyses.
\begin{assumption} \label{assum:input_dynamic}
    \textbf{(Gradient Dynamics).} Let $\{g_{f,t} \in \R\}_{t}$, $\{g_{r, t} \in \R\}_{t}$ be two sequences representing $\nabla_\theta \gL_f$, $\nabla_\theta \gL_r$ at the time step $t$, respectively. 
    We assume the expectations of $g_{f, t}$ and $g_{r, t}$ over time $\mathbb{E}_t \left[g_{f, t}\right] = m G$, $\mathbb{E}_t \left[g_{r, t}\right] = n G$ exist, where $m, n \in [-1, 1]$, and $G$ is a non-negative constant, denoting the largest possible gradient magnitude.
\end{assumption}

Since the update rules in (\ref{eq:base_update}) and (\ref{eq:delta_update}) involve gradients from two distinct distributions, standard convergence properties for stationary inputs do not directly apply.
Nevertheless, the following theorem validates the convergence of both base and delta states in (\ref{eq:base_update}) and (\ref{eq:delta_update}).

\begin{theorem} \label{theorem}
\textbf{(Convergence of Base and Delta States)}. Considering Assumption~\ref{assum:input_dynamic}, the update rules (\ref{eq:base_update}) and (\ref{eq:delta_update}), we use $B_t$, $\Delta_{f, t}$, $\Delta_{r, t}$ to represent the value of state $B$, $\Delta_f$, $\Delta_r$ at the time step $t$, respectively. We have the following asymptotical expectation:
\begin{equation}
\small
\begin{aligned}
\lim_{T \to \infty} B_{(F_f + F_r)T} &= \frac{\beta^{F_r}(1 - \beta^{F_f})m + (1 - \beta^{F_r})n}{1 - \beta^{F_f + F_r}} G, \\
\lim_{T \to \infty} \Delta_{f,(F_f + F_r)T} &= \frac{F_f\beta^{F_f-1}(1-\beta)\left(1-\beta^{F_r}\right)}{\left(1-\beta^{F_f}\right)\left(1-\beta^{F_f+F_r}\right)} (m-n)G, \\
\lim_{T \to \infty} \Delta_{r,(F_f + F_r)T} &= \frac{F_r\beta^{F_r-1}(1-\beta)\left(1-\beta^{F_f}\right)}{\left(1-\beta^{F_r}\right)\left(1-\beta^{F_f+F_r}\right)} (n-m)G.
\end{aligned} \label{eq:theorem}
\end{equation}

where $F_f$ and $F_r$, denote the forget frequency and retain frequency respectively in Algorithm~\ref{alg:do+_adamw}.
\end{theorem}




The proof is deferred to Appendix \ref{sec:theorem_proof}. According to Theorem \ref{theorem}, we can clearly see the expected magnitude of the base state is an interpolation of the expected gradients from the forgetting and the retaining objectives, while the delta states are closely related to their differences. This is consistent with the motivation of DualOptim+.
Specifically, we highlight the behavior of DualOptim+ under the boundary conditions determined by $m$ and $n$:
\vspace{-0.2cm}
\begin{enumerate}[leftmargin=*]
    \item \textbf{Positive Correlation:} When the expectations of gradients are identical $m = n$, the base state $B$ converges to $mG$, and the delta states $\Delta_f$ and $\Delta_r$ converge to $0$. In this case, \textbf{DualOptim+ acts like Alternate}.
    \vspace{-0.2cm}
    \item \textbf{Negative Correlation:} When the expectations of gradients are negatively proportional with a specific factor $m=-\frac{1-\beta^{F_r}}{\beta^{F_r}(1 - \beta^{F_f})}n$, the base state $B$ converges to $0$. In this case, \textbf{DualOptim+ acts like DualOptim}.
\end{enumerate}    
\vspace{-0.2cm}


\begin{figure*}[tb]
    \centering
    \subfigure[Update Similarity]{
    \includegraphics[width=0.35\textwidth]{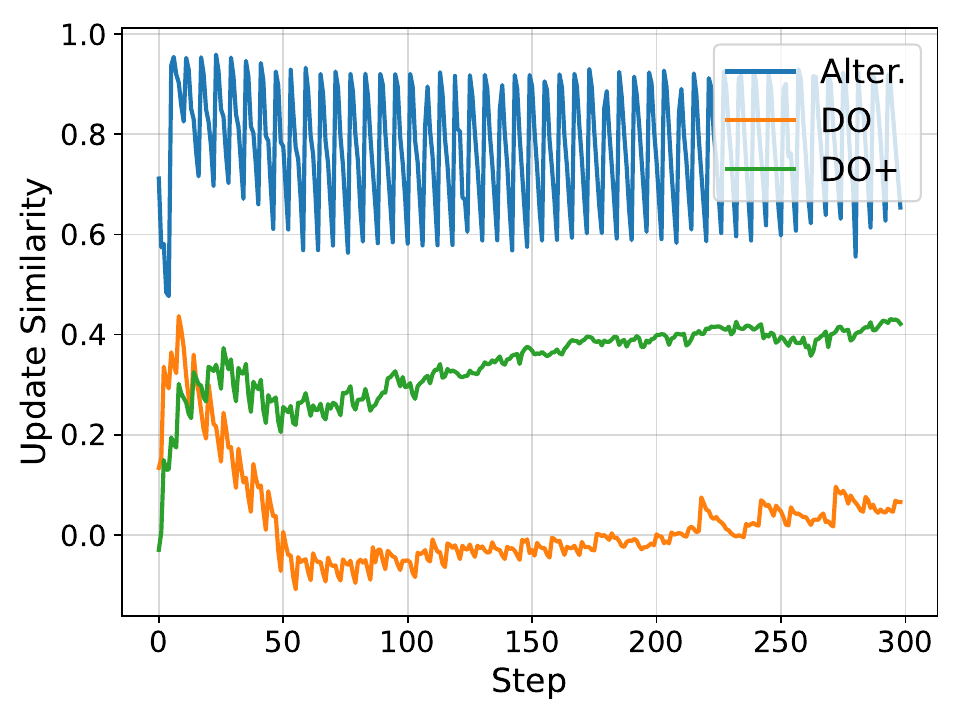}
    }
    \subfigure[Gradient Similarity]{
    \includegraphics[width=0.35\textwidth]{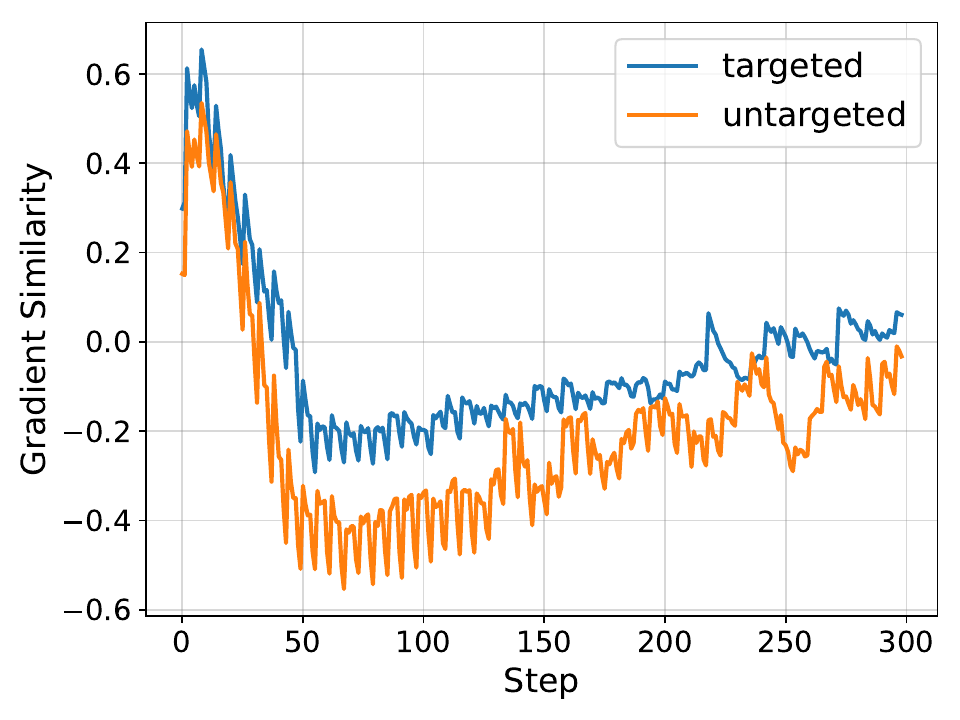}
    }
    \vspace{-0.5em}
    \caption{Comparison of cosine similarity over time steps. \textbf{(a)} Similarity between the update terms for forgetting and retaining of different methods using targeted unlearning objective \citep{yuan2024closer}. \textbf{(b)} Similarity between forgetting and retaining gradients of targeted and untargeted unlearning objectives. For better visualization, the similarity is calculated based on the exponential moving average (EMA) of the gradients with a factor of 0.9. Note that the results are collected based on the forget10 task of TOFU \citep{mainitofu}. The model is TOFU-finetuned Phi-1.5 \citep{li2023textbooks}, and the optimizer is AdamW.} 
    \label{fig:numerical_results}
    \vspace{-1em}
\end{figure*}

\textbf{Numerical Analysis.} We conduct numerical analysis to validate the motivation of DualOptim+.
In Figure~\ref{fig:numerical_results} (a), we show the cosine similarly between the two consecutive updates from the forgetting term and the retaining term for different algorithms. Alternate uses the shared optimizer term, so the similarity is high. By contrast, DualOptim fully decouples these optimizer states and has low similarity. The similarity for DualOptim+ is between Alternate and DualOptim, consistent with our motivation.

Moreover, as illustrated in Figure~\ref{fig:numerical_results} (b), the cosine similarity between the forgetting and retaining gradients exhibits significant volatility throughout the process of both targeted and untargeted unlearning. This instability suggests that the relationship between the two objectives is highly dynamic, further underscoring the necessity of an adaptive mechanism to transition between shared and decoupled states.



%% file: Section/experiments.tex
\section{Experiments}
To demonstrate the effectiveness of DualOptim+ in LLM unlearning, we first define the evaluation metrics employed and then present extensive experimental results on unlearning, safety alignment, and multi-task learning tasks. In addition, a comprehensive ablation study is conducted for further analysis. Detailed implementation settings are provided in Appendix \ref{sec:imple}.

\subsection{Evaluation Metrics}
To comprehensively evaluate the performance of the unlearned model, we adopt the evaluation framework established in \citet{yuan2024closer}, focusing on \textbf{Model Utility (MU)} and \textbf{Forget Efficacy (FE)}, which are calculated over the retain and forget sets, respectively. These aggregated metrics are based on six primary indicators, including ROUGE (R), Probability (P), Truth Ratio (TR), Token Entropy (TE), Cosine Similarity (CS), and Entailment Score (ES), except that TE is excluded from the calculation of FE.

However, standard FE primarily measures the discrepancy between the model's output and the ground truth in the forget set. While suitable for untargeted unlearning, this is insufficient for targeted unlearning scenarios where the model is expected to provide a specific type of reject response. To address this, we introduce \textbf{Targeted Forget Efficacy (TFE)}, which quantifies the similarity between the model's output and predefined rejection templates (e.g., ``Sorry, I don't know"). Given that rejection responses are randomly sampled from these templates, providing no fixed ground truth, we define TFE as the harmonic mean of Token Entropy (TE) and Entailment Score (ES) relative to the rejection response. Consequently, we rename the original Forget Efficacy as \textbf{Untargeted Forget Efficacy (UFE)}. In addition, we list some examples and their corresponding UFEs and TFEs in Appendix \ref{sec:example}.

As a result, when evaluating the forget efficacy,  we use the average of TFE and UFE for targeted unlearning and merely UFE for untargeted unlearning.
We define the \textbf{Overall Performance (OVR)} as the average of forget efficacy and model utility, so  we have $\mathrm{OVR} = 0.25 \times (\mathrm{TFE} + \mathrm{UFE}) + 0.5 \times \mathrm{MU}$ for targeted unlearning and $\mathrm{OVR} = 0.5 \times (\mathrm{UFE} + \mathrm{MU})$ for untargeted unlearning.

\subsection{Unlearning in Fictitious Scenario} \label{sec:tofu}
\begin{table*}[ht]
    \centering
    \caption{Performance comparison of different methods on TOFU-finetuned \textbf{Phi 1.5} and \textbf{Llama 2}. \textbf{IDK+GD} (targeted) and \textbf{ME+GD} (untargeted) are adopted as the loss functions for unlearning. The results include Untargeted Forget Efficacy (UFE), Targeted Forget Efficacy (TFE), Model Utility (MU) and the Overall Performance (OVR) for forget 1\%, 5\% data, and 10\% data. Note that the reported results are the average of the results obtained from $5$ runs with different forget sets.} 
    \label{tab:tofu_idk_me}
    \tabcolsep=0.4em
    \small
    \begin{tabular}{c|c|cccc|cccc|cccc}
        \toprule[1.2pt]
        \multirow{3}{*}{\textbf{Loss}} & \multirow{3}{*}{\textbf{Method}} & \multicolumn{12}{c}{\textbf{Phi 1.5}}\\
        ~&~ & \multicolumn{4}{c|}{forget 1\% data} & \multicolumn{4}{c|}{forget 5\% data} & \multicolumn{4}{c}{forget 10\% data} \\
        ~&~ & UFE $\uparrow$ & TFE  $\uparrow$ &  MU $\uparrow$ & OVR $\uparrow$& UFE $\uparrow$ & TFE  $\uparrow$ &  MU $\uparrow$ & OVR $\uparrow$ & UFE $\uparrow$ & TFE  $\uparrow$ &  MU $\uparrow$ & OVR $\uparrow$ \\
        \midrule[1pt]
        \multirow{6}{*}{\rotatebox{90}{IDK+GD}}
         & Joint & \textbf{78.11}& 45.45 & 18.61 & \cellcolor{gray!30}40.19
         & \textbf{72.55} & 58.32 & 36.26 & \cellcolor{gray!30}50.85 
         & \textbf{71.65} & 64.39 & 33.92 & \cellcolor{gray!30}50.97\\
        ~& Alternate &  73.35 & 62.49 & 48.14 & \cellcolor{gray!30}58.03
         & 67.73 & 64.30 & 47.81 & \cellcolor{gray!30}56.91
         & 65.82 & 64.46 & 49.54 & \cellcolor{gray!30}57.34\\
        ~& DO & 74.75 & 63.51 & 46.46& \cellcolor{gray!30}57.80
         & \ul{68.49} & 64.34 & 49.50 & \cellcolor{gray!30}57.96
         &\ul{65.87} & \textbf{66.84} & 50.25 & \cellcolor{gray!30}58.30\\
        ~& DO 8bit & \ul{75.58}& 61.23& 46.73 & \cellcolor{gray!30}57.57
         & 68.34 & 64.33 & 48.41 & \cellcolor{gray!30}57.37
         & 65.81 & 65.60 & 50.38 & \cellcolor{gray!30}58.05\\
        ~& \textbf{DO+} & 75.51 & \ul{67.85} & \textbf{47.69} & \cellcolor{gray!30}\textbf{59.69}
         & 67.63 & \textbf{67.60} & \textbf{51.52} & \cellcolor{gray!30}\textbf{59.57}
         & 65.42 & \ul{66.50} & \textbf{51.32} & \cellcolor{gray!30}\textbf{58.64}\\
        ~& \textbf{DO+ 8bit} & 73.69 & \textbf{68.36} & \ul{47.53} & \cellcolor{gray!30}\ul{59.28} 
         & 67.56 & \ul{65.94} & \ul{50.36} & \cellcolor{gray!30}\ul{58.55}
         & 65.26 & 65.59& \ul{51.30} & \cellcolor{gray!30}\ul{58.36}\\
        \midrule[1pt]
        \multirow{6}{*}{\rotatebox{90}{ME+GD}} 
         & Joint & \textbf{95.41} & -- & 11.45 & \cellcolor{gray!30}53.43
         & 91.32 & -- & 33.87 & \cellcolor{gray!30}62.60
         & 91.10 & -- & 36.88 & \cellcolor{gray!30}63.99\\
        ~& Alternate & 91.46 & -- & 45.78 & \cellcolor{gray!30}68.62
         & 92.30 & -- & 49.73& \cellcolor{gray!30}71.02 
         & 91.96 & -- & 48.48 & \cellcolor{gray!30}70.22\\
        ~& DO & 92.79 & --& 45.26  & \cellcolor{gray!30}69.03
         & 91.97 & --& 51.73 & \cellcolor{gray!30}71.86
         & 92.39 & -- & 49.23 & \cellcolor{gray!30}70.81\\
        ~& DO 8bit & 92.83 & -- & 45.75 & \cellcolor{gray!30}69.29
         & \textbf{93.12} & --& \textbf{50.99} & \cellcolor{gray!30}\textbf{72.06}
         & 92.32 & -- & 48.04 & \cellcolor{gray!30}70.19\\
        ~& \textbf{DO+} & \ul{93.92}& -- & \textbf{46.19} & \cellcolor{gray!30}\textbf{70.06}
         & \ul{93.07} & -- & \ul{50.87}& \cellcolor{gray!30}\ul{71.97}
         & \ul{92.46} & -- & \textbf{50.32} & \cellcolor{gray!30}\textbf{71.39}\\
        ~& \textbf{DO+ 8bit} & 92.48 & -- & \ul{46.16} & \cellcolor{gray!30}\ul{69.32}
         & 93.13 & --& 50.55 & \cellcolor{gray!30}71.84
         & \textbf{92.78} & -- & 49.96 & \cellcolor{gray!30}\ul{71.37}\\
        \bottomrule[1.2pt]
        \toprule[1.2pt]
        \multirow{3}{*}{\textbf{Loss}} & \multirow{3}{*}{\textbf{Method}} & \multicolumn{12}{c}{\textbf{Llama 2}}\\
        ~&~ & \multicolumn{4}{c|}{forget 1\% data} & \multicolumn{4}{c|}{forget 5\% data} & \multicolumn{4}{c}{forget 10\% data} \\
        ~&~ & UFE $\uparrow$ & TFE  $\uparrow$ &  MU $\uparrow$ & OVR $\uparrow$& UFE $\uparrow$ & TFE  $\uparrow$ &  MU $\uparrow$ & OVR $\uparrow$ & UFE $\uparrow$ & TFE  $\uparrow$ &  MU $\uparrow$ & OVR $\uparrow$ \\
        \midrule[1pt]
        \multirow{6}{*}{\rotatebox{90}{IDK+GD}}
         & Joint & \textbf{85.96} & \ul{59.50} & 38.62 & \cellcolor{gray!30}55.68
         & \textbf{80.29} & \ul{68.60} & 55.63 & \cellcolor{gray!30}65.04
         & \textbf{78.08} & \textbf{71.25} & 57.15 & \cellcolor{gray!30}65.91\\
        ~& Alternate & 81.97 &\textbf{66.41} & 73.93 & \cellcolor{gray!30}74.06
        &75.79 & \textbf{70.27} & 73.93 & \cellcolor{gray!30}73.48
        & 74.16 & \ul{68.15} & 73.98 & \cellcolor{gray!30}72.57\\
        ~& DO & \ul{83.12} & \textbf{66.41} & 73.83 & \cellcolor{gray!30}74.30
        &\ul{76.58} & \textbf{70.27} & 73.64& \cellcolor{gray!30}73.53
        & \ul{74.62} & \ul{68.15} & 73.15 & \cellcolor{gray!30}72.27\\
        ~& DO 8bit & 83.11 &\textbf{66.41} & 73.77 & \cellcolor{gray!30}74.27
        & 76.42 & \textbf{70.27} & 73.71 & \cellcolor{gray!30}73.53
        & \ul{74.62} & \ul{68.15} & 73.38 & \cellcolor{gray!30}72.38\\
        ~& \textbf{DO+} & 78.76 & \textbf{66.41} & \textbf{77.40} & \cellcolor{gray!30}\textbf{74.99}
        & 75.11 & \textbf{70.27} & \ul{75.91} & \cellcolor{gray!30}\ul{74.30}
        & 73.91 & \ul{68.15} & \ul{75.46} & \cellcolor{gray!30}\ul{73.25}\\
        ~& \textbf{DO+ 8bit} &79.05 & \textbf{66.41} & \ul{76.53} & \cellcolor{gray!30}\ul{74.63}
        & 75.04 & \textbf{70.27} &\textbf{76.27} & \cellcolor{gray!30}\textbf{74.42}
        & 73.49 & \ul{68.15} & \textbf{76.14} & \cellcolor{gray!30}\textbf{73.48}\\
        \midrule[1pt]
        \multirow{6}{*}{\rotatebox{90}{ME+GD}} 
         & Joint & 95.89 & -- & 59.70 & \cellcolor{gray!30}77.80
         & \textbf{97.65} & -- & 57.15 & \cellcolor{gray!30}77.40
         & \textbf{97.66} & --& 60.63 & \cellcolor{gray!30}79.15\\
        ~& Alternate & 97.25 & -- & 74.89 & \cellcolor{gray!30}86.07
        & \ul{97.07} & -- & 75.64 & \cellcolor{gray!30}86.36
        & \ul{96.86} & -- & 75.34 & \cellcolor{gray!30}86.10\\
        ~& DO & 97.46 & -- & 75.06 & \cellcolor{gray!30}86.26
        & 96.66 & -- &75.90 & \cellcolor{gray!30}86.28
        & 96.78& --& \ul{75.60} & \cellcolor{gray!30}\ul{86.19}\\
        ~& DO 8bit & \ul{97.76} & -- & \ul{75.55} & \cellcolor{gray!30}\ul{86.66}
        & 96.74 & -- & 75.52 & \cellcolor{gray!30}86.13
        & 96.57 & -- & 75.44 & \cellcolor{gray!30}86.01\\
        ~& \textbf{DO+} & \textbf{97.88} & -- & \textbf{75.82} & \cellcolor{gray!30}\textbf{86.85}
        & 96.91 & -- & \ul{76.09} & \cellcolor{gray!30}\textbf{86.50}
        & 96.85 & -- & \textbf{75.86} & \cellcolor{gray!30}\textbf{86.35}\\
        ~& \textbf{DO+ 8bit} & 97.50 & -- &  75.01 & \cellcolor{gray!30}86.26
        & 96.69 & -- & \textbf{76.16} & \cellcolor{gray!30}\ul{86.43}
        & 96.78 & -- & 75.52 & \cellcolor{gray!30}86.15\\
        \bottomrule[1.2pt]
    \end{tabular}
    \vspace{-1em}
\end{table*}

We start with evaluating our method on the standard TOFU benchmark \citep{mainitofu}. TOFU simulates an ideal scenario with full data access, featuring 200 fictitious authors (20 QA pairs each). It includes three tasks: forget01, forget05, and forget10, targeting the removal of 1\%, 5\%, and 10\% of the data, respectively. The remaining data constitutes the retain set. Additionally, ``Real Authors'' and ``World Facts'' sets are used to evaluate general knowledge utility. We employ Phi-1.5-1.3B \citep{li2023textbooks} and Llama-2-7B \citep{touvron2023llama} released by TOFU as the target models, which has been fine-tuned on the constructed data to ensure it can exactly gives answers to questions in TOFU.

To evaluate the effectiveness of \textbf{DualOptim+ (DO+)}, we compare it against three baselines: \textbf{Joint}, \textbf{Alternate}, and \textbf{DualOptim (DO)} \citep{zhong2025dualoptim}. Notably, DO and DO+ require additional memory for optimizer states, specifically $2\times$ and $3\times$ the memory consumption of a single standard Adam, respectively. To mitigate this overhead, we introduce their 8-bit versions, \textbf{DO 8bit} and \textbf{DO+ 8bit}, implemented using the \texttt{bitsandbytes} library to quantize optimizer states. These 8-bit variants significantly reduce memory requirements to $1/2$ and $3/4$ that of a single Adam. To quantify this, we compare the running time and memory consumption of the evaluated methods in Appendix \ref{sec:compare_time_memory}.

The results on IDK+GD (targeted) \citep{mainitofu} and ME+GD (untargeted) \citep{yuan2024closer} loss functions are reported in Table \ref{tab:tofu_idk_me}.
We can observe that Joint suffers from excessive forgetting, leading to a significant degradation in model utility. In contrast, DualOptim+ (DO+) achieves the best overall performance in most cases, with its superiority primarily stemming from its ability to preserve model utility while effectively unlearning. 
Additionally, DO+ yields a larger performance gain over DO in targeted unlearning compared to untargeted unlearning, as the forgetting and retaining objectives are less conflicted in the targeted setting, allowing DO+ to better leverage its optimization advantages as an intermediate state between fully shared state and fully decoupled states. 
Moreover, as illustrated in Figure~\ref{fig:curve} of Appendix~\ref{sec:curve_metric}, DO+ demonstrates consistently better and more stable unlearning performance over time steps compared to the baselines.
Regarding efficiency, the proposed DO 8bit and DO+ 8bit variants consume only $1/4$ of the memory required by their standard 32-bit counterparts (DO and DO+) without compromising performance. 

Additionally, the results on DPO+GD \citep{rafailov2023direct} and NPO+GD \citep{zhang2024negative} loss functions are presented in Table \ref{tab:tofu_dpo_npo} of Appendix \ref{sec:tofu_dpo_npo}.
To further evaluate the efficacy of our method with parameter-efficient fine-tuning techniques, we report the results with LoRA \citep{hu2022lora} in Table \ref{tab:tofu_idk_me_lora} of Appendix \ref{sec:tofu_lora}.
 These observations are consistent with that in Table \ref{tab:tofu_idk_me}, indicating the effectiveness of DualOptim+ in broad scenarios.

\subsection{Unlearning in Real-world Scenario} \label{sec:real_world}
\begin{table*}[t]
    \centering
    \caption{Performance comparison of different methods on \textbf{Llama 3} for unlearning real-world data. \textbf{IDK+GD} (targeted) and \textbf{ME+GD} (untargeted) are adopted as the loss functions for unlearning. The results for unlearning tasks include Untargeted Forget Efficacy (UFE), Targeted Forget Efficacy (TFE), Model Utility (MU) and the Overall Performance (OVR). The average metrics (AVG) on downstream tasks are calculated.  Note that the reported results are the average of the results obtained from $3$ runs using different random seeds.} 
    \label{tab:real_world}
    \tabcolsep=0.4em
    \small
    \begin{tabular}{c|c|cccc|cccccc}
        \toprule[1.2pt]
        \multirow{3}{*}{\textbf{Loss}} & \multirow{3}{*}{\textbf{Method}} & \multicolumn{10}{c}{\textbf{Llama 3}}\\
        ~&~ & \multicolumn{4}{c|}{\textbf{Unlearning Task}} & \multicolumn{6}{c}{\textbf{Downstream Tasks}} \\
        ~&~ & UFE $\uparrow$ & TFE  $\uparrow$ &  MU $\uparrow$ & OVR $\uparrow$& ARC-c $\uparrow$& MMLU $\uparrow$& TruthfulQA $\uparrow$& TriviaQA $\uparrow$& GSM8K $\uparrow$& AVG $\uparrow$\\
        \midrule[1pt]
         ~ & Initial & 30.55 & -- & 61.45 & \cellcolor{gray!30}46.00 & 55.38 & 64.59 & 37.33 & 50.93 & 76.12 & \cellcolor{gray!30}56.87\\
        \midrule[1pt]
        \multirow{6}{*}{\rotatebox{90}{IDK+GD}} 
         & Joint & \ul{85.54} & \textbf{72.96} & 27.38 & \cellcolor{gray!30}53.32 & 46.79 & 62.85 & 33.41 & 7.58 & 74.32 & \cellcolor{gray!30}44.99\\
        ~& Alternate& 85.49 & \ul{69.95} & \ul{29.19} & \cellcolor{gray!30}\ul{53.45} & 49.77 & 63.31 & 35.62 & \ul{12.71} & 74.15 & \cellcolor{gray!30}47.11\\
        ~& DO & 85.25 & 69.60 & 28.06 & \cellcolor{gray!30}52.73 & 48.47 & 63.20 & 35.29 & 10.33 & 72.35 & \cellcolor{gray!30}45.93\\
        ~& DO 8bit & 85.28 & 69.60 & 27.68 & \cellcolor{gray!30}52.56 & 48.75 & 63.08 & 35.05 & 11.56 & 72.35 & \cellcolor{gray!30}46.16\\
        ~& \textbf{DO+} & \textbf{85.72} & 69.94 & 27.96 & \cellcolor{gray!30}52.90 & \ul{50.85} & \ul{64.43} & \ul{36.35} & 11.17 & \textbf{76.02} & \cellcolor{gray!30}\ul{47.77}\\
        ~& \textbf{DO+ 8bit} & 85.47 & 69.59 & \textbf{33.36} & \cellcolor{gray!30}\textbf{55.45} & \textbf{52.56} &\textbf{64.51} & \textbf{36.80} & \textbf{17.86} & \ul{75.21} & \cellcolor{gray!30}\textbf{49.39}\\
        \midrule[1pt]
        \multirow{6}{*}{\rotatebox{90}{ME+GD}} 
         & Joint & \textbf{97.97} & -- & 24.53 & \cellcolor{gray!30}61.25 & 43.29 & 63.46 & 25.05 & 29.61 & 62.34 & \cellcolor{gray!30}44.75\\
        ~& Alternate & 97.75 & -- & 35.23 & \cellcolor{gray!30}66.49 &48.66 & 64.00 & 25.38 & \textbf{38.30} & 63.68 & \cellcolor{gray!30}48.00\\
        ~& DO & 97.67 &--& 37.51 & \cellcolor{gray!30}67.60 & 45.36 & 63.27 & 25.34 & \ul{37.45} & 37.07 & \cellcolor{gray!30}41.69\\
        ~& DO 8bit & 97.67 & -- & 35.42 & \cellcolor{gray!30}66.55 & 47.95 & 63.67 & 25.95 & 34.20 & 47.58 & \cellcolor{gray!30}43.87\\
        ~& \textbf{DO+} & \ul{97.85} & -- & \ul{48.40}& \cellcolor{gray!30}\ul{73.13} & \textbf{56.52} & \textbf{64.16} & \textbf{34.80} & 28.08 & \textbf{73.44} & \cellcolor{gray!30}\textbf{51.40} \\
        ~& \textbf{DO+ 8bit} & 97.77 & --& \textbf{49.29}& \cellcolor{gray!30}\textbf{73.52} & \ul{56.45} & \ul{64.14} & \ul{31.29} & 29.22 & \ul{72.38} &\cellcolor{gray!30}\ul{50.70}\\
        \bottomrule[1.2pt]
    \end{tabular}
\end{table*}
We consider a realistic scenario where the unlearning is performed without access to the original training data. Following \citet{liu2024learning,yuan2024closer}, we identify real-world individuals memorized by Llama-3-8B-Instruct \citep{grattafiori2024llama}. We select 20 individuals as unlearning targets, generating the forget set using Llama 3's own responses to 20 questions per person. A neighbor set of 40 additional individuals is selected as the retain set: 20 of which are used for regularization during unlearning, while the remaining 20 are used to evaluate Model Utility. Furthermore, general capability is evaluated via five downstream tasks: ARC-c \citep{Clark2018ThinkYH}, MMLU \citep{hendrycks2021measuring}, TruthfulQA \citep{Lin2021TruthfulQAMH}, TriviaQA \citep{joshi-etal-2017-triviaqa}, and GSM8K \citep{Cobbe2021TrainingVT}. Joint, Alternate, DualOptim (DO), DualOptim+ (DO+), and the 8bit variants of DO and DO+ are evaluated in the experiment. IDK+GD and ME+GD are adopted as the loss functions for targeted and untargeted unlearning, respectively.

As shown in Table \ref{tab:real_world}, DO+ and its 8-bit variant (DO+ 8bit) consistently achieves effective forgetting while maintaining superior machine utility on the retain set and downstream tasks. Surprisingly, DO underperforms the simpler Alternate approach, particularly in model utility and downstream tasks, suggesting that the fully decoupled framework requires the additional refinements present in DO+ to effectively handle the complexities of real-world data unlearning. Furthermore, compared with untargeted unlearning, the results indicate that targeted unlearning typically leads to a collapse in model utility compared to the initial state, dropping from $61.45$ to roughly half that value. This phenomenon is likely to arise from a rigid mapping to specific "I don't know" responses in targeted objective, leading to catastrophic interference with the model's internal representations. In contrast, the untargeted objective merely aims to suppress the ground-truth response through entropy maximization, imposing a ``softer'' optimization constraint that induces a more marginal representational shift and better preserves the model's underlying reasoning logic.

\subsection{Safety Alignment} \label{sec:safety_align}
In this subsection, we extend our evaluation to safety alignment \citep{Bai2022TrainingAH}, which can be framed as a specialized unlearning task. The objective is to eliminate unsafe knowledge while preserving model utility. Following the experimental setup of \citet{bianchisafety}, our initial model is Llama-3-8B-Instruct fine-tuned on Alpaca \citep{alpaca} to ensure robust instruction-following capabilities. We then simulate the unlearning process by tuning the model on 20,000 Alpaca instructions combined with 2,000 safety instructions; here, the safety and Alpaca instructions serve as the forget and retain sets, respectively. Follow \citet{bianchisafety}, we perform standard SFT on the mixed dataset, which is equivalent to targeted unlearning. The objectives for untargeted unlearning are not considered here, since the objective does not meet the requirements of the task, i.e., reject replying to unsafe queries.

To evaluate the model's harmlessness, we utilize the same collection of harmful instruction datasets as that in \citet{bianchisafety}, including I-MaliciousInstructions (I-Mali) \citep{alpaca}, I-CoNa \citep{fanton-etal-2021-human}, I-Controversial (I-Cont) \citep{bianchisafety}, and Q-Harm \citep{Bai2022TrainingAH}. Response safety is assessed using Llama-Guard-2-8B \citep{metallamaguard2}, with the final safety rate reported as the primary metric. Model utility is measured following the methodology described in Sec. \ref{sec:real_world}. Additionally, we calculate the over-refusal rate on XSTest \citep{rottger-etal-2024-xstest} to identify exaggerated safety behaviors. We compare all methods previously discussed in Sec. \ref{sec:tofu} and \ref{sec:real_world}.

\begin{table*}[h]
    \centering
    \caption{Performance comparison of different methods on \textbf{Alpaca-finetuned Llama 3} for safety alignment. The averages (AVG) of the metrics on safety and utility tasks are calculated, respectively. XSTest is used to evaluate the over-refusal rate. Note that the reported results are the average of the results obtained from $3$ runs using different random seeds.} 
    \label{tab:safe_align}
    \tabcolsep=0.3em
    \small
    \resizebox*{\textwidth}{!}{
    \begin{tabular}{c|ccccc|cccccc|c|c}
        \toprule[1.2pt]
        \multirow{3}{*}{\textbf{Method}} & \multicolumn{13}{c}{\textbf{Alpaca-Llama 3}}\\
        ~ & \multicolumn{5}{c|}{\textbf{Safety}} & \multicolumn{6}{c|}{\textbf{Utility}} & \multirow{2}{*}{\textbf{OVR} $\uparrow$}  & \multirow{2}{*}{\textbf{XSTest} $\downarrow$}\\
        ~ & I-Mali $\uparrow$ & I-CoNa  $\uparrow$ &  I-Cont $\uparrow$ & Q-Harm $\uparrow$&  AVG $\uparrow$& ARC-c $\uparrow$& MMLU $\uparrow$& TruthfulQA $\uparrow$& TriviaQA $\uparrow$& GSM8K $\uparrow$& AVG $\uparrow$&~ & \\
        \midrule[1pt]
         Initial & 28.00 & 38.76 & 55.00 & 64.00 & \cellcolor{gray!30}46.44 & 45.56 & 52.53 & 29.74 & 12.11 & 13.12 & \cellcolor{gray!30}30.61 & \cellcolor{gray!30}33.56  &\cellcolor{gray!30}0.40\\
        \midrule[1pt]
        Joint & 94.67 & 96.63 & \textbf{97.50} & 97.00 & \cellcolor{gray!30}96.45 & 47.04 & 51.63& 33.74 & 12.18 & 14.10 & \cellcolor{gray!30}31.74 & \cellcolor{gray!30}54.84 & \cellcolor{gray!30}\textbf{28.00}\\
        Alternate& \textbf{97.00} & 97.38 & \textbf{97.50} & \textbf{99.67} & \cellcolor{gray!30}\textbf{97.89} & 46.81 & 50.83 & \textbf{34.60} & 13.83 & 12.94 & \cellcolor{gray!30}31.80 & \cellcolor{gray!30}55.67 & \cellcolor{gray!30}29.20\\
        DO & 95.67 & \ul{97.94} & \textbf{97.50} & 99.30 & \cellcolor{gray!30}97.61 & \textbf{47.36} & 50.13 & 33.58  & \ul{14.02} & 12.99 & \cellcolor{gray!30}31.62 & \cellcolor{gray!30}\ul{55.76} & \cellcolor{gray!30}30.27\\
        DO 8bit & \ul{96.00} & \textbf{98.31} &\ul{95.50} & 99.00 & \cellcolor{gray!30}97.20 & 47.10 & 50.78& 33.58 & 13.82 & 12.96 & \cellcolor{gray!30}31.65 & \cellcolor{gray!30}54.66 & \cellcolor{gray!30}28.27\\
        \textbf{DO+} & \ul{96.00} & 97.56 & \textbf{97.50} & 98.67 & \cellcolor{gray!30}97.43 & \ul{47.27} & \textbf{51.89}& 32.25 & \textbf{15.39} & \ul{14.23} & \cellcolor{gray!30}\textbf{32.81} & \cellcolor{gray!30}\textbf{56.45} & \cellcolor{gray!30}\ul{28.13}\\
        \textbf{DO+ 8bit} & \ul{96.00} & \textbf{98.31} & \textbf{97.50} & \ul{99.33} & \cellcolor{gray!30}\ul{97.79} & 46.93 & \ul{51.66} & \ul{34.47} & 13.71 & \textbf{14.73} & \cellcolor{gray!30}\ul{32.30} & \cellcolor{gray!30}55.61 & \cellcolor{gray!30}28.27\\
        \bottomrule[1.2pt]
    \end{tabular}}
\end{table*}

As shown in Table \ref{tab:safe_align}, DO+ achieves the superior safety-utility trade-off among the evaluated methods. While all unlearning-based approaches significantly improve the model's safety metrics compared with the initial state, DO+ distinguishes itself by attaining the highest overall performance (56.45\%) and the highest average utility score (32.81\%). This indicates that DO+ is particularly effective at erasing harmful knowledge while preserving the model's core capabilities. Furthermore, DO+ maintains a relatively low over-refusal rate of 28.13\% on the XSTest benchmark, which is notably lower than that of the standard DO (30.27\%) and Alternate (29.20\%) methods. These results suggest that DO+ not only successfully aligns the model with safety requirements but also effectively mitigates the common pitfall of exaggerated safety behaviors.

\subsection{Multi-task Learning}
Our method can be easily extended to multi-task learning tasks. For $N$ tasks, we need to maintain one base state and $N$ delta states. Specifically, we finetune Llama-2-7B on three different tasks, i.e., Py150 (code) \citep{lu2021codexglue}, ScienceQA (science) \citep{mishra-etal-2022-numglue}, NumGLUE-cm (math) \citep{lu2022learn}. The datasets are collected from TRACE \citep{wang2023trace}, and we only evaluate DO 8bit and DO+ 8bit to reduce memory consumption. 

The results listed in Table \ref{tab:mtl} indicate that our method is still effective in the context of multi-task learning. Note that in the context of unlearning, the severity of gradient conflicts gives DO a distinct advantage. Conversely, in multi-task learning, where these conflicts are less pronounced, the gains from DO are negative, whereas DO+ continues to deliver a substantial boost in performance.

\begin{table}[H]
    \centering
    \vspace{-0.5em}
    \caption{Performance in multi-task learning task. We finetune Llama-2-7B on three different tasks, i.e., Py150 (code), ScienceQA (science) , NumGLUE-cm (math).}
    \small
    \begin{tabular}{c|c c c c}
    \toprule[1.2pt]
       ~ & Py150 & ScienceQA & NumGLUE-cm & AVG\\
    \midrule[1pt]
       Joint & 61.09 & 92.40 & 41.67 & \cellcolor{gray!30}65.05\\
       Alternate & 60.74 & 92.05 & 42.86 & \cellcolor{gray!30}65.22 \\
       DO 8bit & 60.36 & 92.25 & 40.48 & \cellcolor{gray!30}64.36 \\
       \textbf{DO+ 8bit} & 60.87 & 91.75& 48.81 & \cellcolor{gray!30}\textbf{67.14} \\
    \bottomrule[1.2pt]
    \end{tabular}
    \vspace{-1em}
    \label{tab:mtl}
\end{table}

\subsection{Ablation Study}
In this subsection, we conduct ablation study on DualOptim+ for further analysis. If not specified, all experiments are conducted based on forgetting 5\% data of TOFU on Phi1.5 by using IDK+GD loss function. AdamW is adopted as the optimizer.

\textbf{Update timing of base state.} In Table \ref{tab:priority_base}, we evaluate the impact of the update timing for the base state. Specifically, we compare the performance when the base state is updated at different stages of the optimization step. The results demonstrate that DualOptim+ achieves optimal performance when the base state is updated after the parameter update. This strategy ensures that the delta states are calculated against a stable and lagged reference, suppressing oscillations during optimization.
\begin{table}[H]
    \centering
    \caption{Unlearning performance when the base state is updated at different stages.}
    \small
    \begin{tabular}{c|c c c c}
    \toprule[1.2pt]
       Stage & UFE $\uparrow$ & TFE $\uparrow$ & MU $\uparrow$ & OVR $\uparrow$\\
    \midrule[1pt]
       Before $\Delta$ & 67.54 & 66.55 & 50.82 & \cellcolor{gray!30}58.93\\
       After $\Delta$ & \textbf{68.24} & 64.38 & 49.97 & \cellcolor{gray!30}58.14\\
       After $\vtheta$ & 67.63 & \textbf{67.60} & \textbf{51.52} & \cellcolor{gray!30}\textbf{59.57}\\
    \bottomrule[1.2pt]
    \end{tabular}
    \vspace{-1em}
    \label{tab:priority_base}
\end{table}

\textbf{Updating rule of base state.}
In Table \ref{tab:update_rule_base}, we evaluate the performance of DualOptim+ when the base state is updated via different components, specifically comparing raw gradients ($\vg$) against the difference between gradients and bias-corrected delta states ($\vg - \widehat{\Delta}$). The results demonstrate that updating the base state directly with gradients yields the best performance, reinforcing its role in capturing the shared representation between tasks.
\begin{table}[H]
    \centering
    \vspace{-0.5em}
    \caption{Unlearning performance when the base state is updated by different components.}
    \small
    \begin{tabular}{c|c c c c}
    \toprule[1.2pt]
       Updated by & UFE $\uparrow$ & TFE $\uparrow$ & MU $\uparrow$ & OVR $\uparrow$\\
    \midrule[1pt]
       $\vg$ & 67.63 & \textbf{67.60} & \textbf{51.52} & \cellcolor{gray!30}\textbf{59.57}\\
       $\vg-\widehat{\Delta}$ & \textbf{67.78}&66.74&50.89& \cellcolor{gray!30}59.08\\
    \bottomrule[1.2pt]
    \end{tabular}
    \vspace{-1em}
    \label{tab:update_rule_base}
\end{table}

\textbf{Momentum factors.} To avoid introducing additional hyperparameters, we utilize identical momentum factors ($\beta_1, \beta_2$) for both the base and delta states by default. To evaluate the sensitivity of this configuration, we test two distinct momentum factor sets for AdamW: \textbf{Fast} ($\beta_1=0.9, \beta_2=0.95$, the default) and \textbf{Slow} ($\beta_1=0.99, \beta_2=0.999$). We examine the performance across various combinations of these sets as presented in Table \ref{tab:momentum_factor}. Our results confirm that the default setting, where both states employ the fast momentum factors, achieves the best performance. Notably, applying the slow set to both states leads to a significant performance drop, as the resulting updates are overly conservative for the unlearning task.

\begin{table}[H]
    \centering
    \vspace{-0.5em}
    \caption{Unlearning performance when adopting different momentum factors. We introduce two sets of momentum factors, Fast ($\beta_1=0.9$, $\beta_2=0.95$), Slow ($\beta_1=0.99$, $\beta_2=0.999$). The configuration (F, S) means that the Fast (F) set is used to update the delta states and the Slow (S) set is used to update the base state.}
    \small
    \begin{tabular}{c|c c c c}
    \toprule[1.2pt]
       \makecell{Momentum\\factors} & UFE $\uparrow$ & TFE $\uparrow$ & MU $\uparrow$ & OVR $\uparrow$\\
    \midrule[1pt]
       (F, F) & 67.63 & \textbf{67.60} & \textbf{51.52} & \cellcolor{gray!30}\textbf{59.57}\\
       (F, S) & 67.14& 67.00 & 50.69&\cellcolor{gray!30}58.88\\
       (S, F) & \textbf{67.88}&67.53&51.32&\cellcolor{gray!30}59.51\\
       (S, S) & 64.90&64.27&50.60&\cellcolor{gray!30}57.59\\
    \bottomrule[1.2pt]
    \end{tabular}
    \vspace{-1em}
    \label{tab:momentum_factor}
\end{table}

\textbf{Quantization.} To decrease memory overhead, we introduce an 8-bit variant of DualOptim+. Table \ref{tab:quantize} evaluates the impact of quantizing specific optimizer states. The results demonstrate that quantization yields significant memory savings with only acceptable performance degradation. Furthermore, we observe only a marginal performance gap when varying which states are quantized. Therefore, to maximize memory efficiency, we adopt the quantization of all optimizer states as our default configuration.
\begin{table}[H]
    \centering
    \vspace{-0.5em}
    \caption{Unlearning performance when quantizing different states.}
    \small
    \begin{tabular}{c|c c c c}
    \toprule[1.2pt]
        Quantize& UFE $\uparrow$ & TFE $\uparrow$ & MU $\uparrow$ & OVR $\uparrow$\\
    \midrule[1pt]
       None & 67.63 & 67.60 & 51.52 & \cellcolor{gray!30}59.57\\
       $B$ & 67.84 & 66.90 & 50.66 & \cellcolor{gray!30}59.02\\
       $\Delta$ & 67.55 & 66.19 & 50.31 & \cellcolor{gray!30}58.59\\
       $B+\Delta$ & 67.56 & 65.94 & 50.36 & \cellcolor{gray!30}58.55\\
    \bottomrule[1.2pt]
    \end{tabular}
    \vspace{-1em}
    \label{tab:quantize}
\end{table}

\textbf{Retain frequency.} In Table \ref{tab:retain_freq}, we investigate the impact of various retain frequencies ($F_r$) on the Alternate, DualOptim (DO), and DualOptim+ (DO+), while maintaining a fixed forget frequency ($F_f = 1$). The results indicate that the optimal overall performance is achieved at $F_r = 5$ for our method. While the sensitivity to $F_r$ is relatively low, we observe a general trade-off: a smaller $F_r$ tends to enhance forget efficacy at the expense of model utility, whereas a larger $F_r$ better preserves utility but slightly diminishes unlearning effectiveness. Notably, the improvement across different $F_r$ for our method is stable and is consistently better than Alternate and DO.
\begin{table}[H]
    \centering
    \vspace{-0.5em}
    \caption{Overall unlearning performance with different retain frequencies $F_r$. The forget frequency is fixed to 1. The total steps of unlearning is 300.} \label{tab:retain_freq}
    \small
    \begin{tabular}{c|c c c c c c}
    \toprule[1.2pt]
       $F_r$ & 1 & 2 & 4 & 5& 9 & 14\\
    \midrule[1pt]
      Alter. & 55.82 & 56.10 & 56.18 &\textbf{56.28} & 55.79 & 55.06\\
      DO & 56.54 & \textbf{59.46} & 57.96 &59.18 & 58.28 & 56.96\\
      \textbf{DO+} & 59.13 & 60.05 & 60.19 &\textbf{61.30} & 60.85 & 58.16\\
    \bottomrule[1.2pt]
    \end{tabular}
    \vspace{-1em}
\end{table}

\textbf{Hyperparameters of Joint.} In Table \ref{tab:joint_hp}, we report the results of a grid search over various forgetting loss weights (with a fixed retaining loss weight of 1) and learning rates for the Joint updating scheme. The results indicate that the Joint method achieves its optimal performance using the default hyperparameters. However, even with tuned hyperparameters, the Joint method consistently underperforms compared to other baselines, particularly in terms of maintaining model utility.
\begin{table}[H]
    \centering
    \vspace{-0.5em}
    \caption{Unlearning performance of Joint with different \textbf{(a)} forgetting loss weights, and \textbf{(b)} different learning rates.} \label{tab:joint_hp}
    \vspace{-0.5em}
    \small
    \subtable[Forgetting Loss Weight]{
    \begin{tabular}{c|c c c c c c}
    \toprule[1.2pt]
       Weight & 0.1 & 0.2 & 0.4 & 0.6& 0.8 & \textbf{1.0}\\
    \midrule[1pt]
      UFE $\uparrow$ & 72.89 & 73.41 & 73.59 & 74.59 & 74.08 & \textbf{74.78}\\
      TFE $\uparrow$& 54.46 & 54.53 & 61.51 & 62.98 & 65.89 & \textbf{66.45} \\
      MU $\uparrow$ & \textbf{42.82} & 41.78 & 37.98 & 38.44 & 37.50 & 36.90\\
      OVR $\uparrow$ & \cellcolor{gray!30}53.25 & \cellcolor{gray!30}52.88 & \cellcolor{gray!30}52.77 & \cellcolor{gray!30}53.61 & \cellcolor{gray!30}53.74 & \cellcolor{gray!30}\textbf{53.76}\\
    \bottomrule[1.2pt]
    \end{tabular}
    \label{tab:joint_weight}
    }
    \vspace{-0.5em}
    \subtable[Learning Rate]{
    \begin{tabular}{c|c c c c c}
    \toprule[1.2pt]
       Weight &5e-6& \textbf{1e-5}& 2e-5 & 4e-5 & 6e-5\\
    \midrule[1pt]
      UFE $\uparrow$ & 72.72 & 74.78 & 75.61 & 80.21 & \textbf{81.14}\\
      TFE $\uparrow$& 61.62 & 66.45 & \textbf{69.37} & 67.94 & 66.55\\
      MU $\uparrow$ & 34.79 & \textbf{36.90} & 34.94 & 32.95 & 28.10\\
      OVR $\uparrow$ & \cellcolor{gray!30}50.98 & \cellcolor{gray!30}\textbf{53.76} & \cellcolor{gray!30}53.72 & \cellcolor{gray!30}53.51 & \cellcolor{gray!30}50.97\\
    \bottomrule[1.2pt]
    \end{tabular}
    \label{tab:joint_lr}
    }
    \vspace{-1em}
\end{table}

%% file: Section/conclusion.tex
\section{Conclusion}
In this work, we proposed a novel optimization framework named DualOptim+ for LLM unlearning that utilizes a base state to capture shared representations between forgetting and retaining objectives, alongside delta states that preserve objective-specific residuals. Our extensive evaluation across diverse unlearning scenarios demonstrates that DualOptim+ provides a more stable and effective trade-off between knowledge erasure and utility preservation than existing methods. In future work, we aim to explore the applicability of our method in more general scenarios.

\section*{Impact Statement}
Our method is a generalizable optimization framework to help achieve a better trade-off when there are multiple conflicting learning objectives. It can be applied in broad scenarios, including machine unlearning, safety alignment, multi-task learning, etc.

\section*{Acknowledgements}

This work is supported by City University of Hong Kong (Project No. 9220132, 9229203).


%% file: Appendix/alg.tex
\section{Pseudo-code of DualOptim+ with Muon} \label{sec:app_do+_muon}
The pseudo-code of DualOptim+ integrated in Muon \cite{jordan2024muon} is shown in Algorithm \ref{alg:do+_muon}.
\begin{algorithm}[H]
   \caption{DualOptim+ with Muon}
   \label{alg:do+_muon}
   \setstretch{1.07}
\begin{algorithmic}[1]
    \STATE {\bfseries Input:} parameter $\vtheta$, learning rate $\eta$, momentum factor $\beta$, forget objective $\gL_f$, retain objective $\gL_r$, total steps $N$, forget frequency $F_f$, retain frequency $F_r$
    \STATE {\bfseries Initialize:} $\vm_{\Delta_f}\leftarrow 0$, $\vm_{\Delta_r}\leftarrow 0$, $\vm_B\leftarrow 0$
    \FOR{$t=1$ \textbf{to} $N$}
        \IF{$t\bmod (F_f + F_r) \leq F_f$} 
            \STATE $\vg,~\vm_{\Delta} \leftarrow \nabla_{\vtheta}{\gL_f}(\vtheta),~\vm_{\Delta_f}$
        \ELSE
            \STATE $\vg ,~\vm_{\Delta}\leftarrow \nabla_{\vtheta}{\gL_r}(\vtheta),~\vm_{\Delta_r}$
        \ENDIF
        \STATE $\vm_{\Delta} \leftarrow \beta\vm_{\Delta} + (\vg - \vm_B)$
        \STATE $\vo \leftarrow \mathtt{NewtonSchulz5}(\vm_B + \vm_{\Delta})$
        \STATE $\vtheta = \vtheta - \eta \vo$
        \STATE $\vm_B \leftarrow \beta\vm_B + \vg$
    \ENDFOR
    \STATE {\bfseries Output:} parameter $\vtheta$
\end{algorithmic}
\end{algorithm}

%% file: Appendix/discussion.tex
\section{Proofs} 

\subsection{Proof of Theorem~\ref{theorem}} \label{sec:theorem_proof}

\begin{proof}

We prove the convergence of $B_t$, $\Delta_f$ and $\Delta_r$ one by one.

\textbf{1. Convergence of $B_t$}

Based on the update rule~(\ref{eq:base_update}), we have the following equation:
\begin{equation}
    B_{(F_f + F_r)T} = \beta^{F_f + F_r} B_{(F_f + F_r)(T-1)} + (1 - \beta) \left(\sum_{t = 1}^{F_f} \beta^{F_f + F_r - t} g_{f, (F_f + F_r)(T-1) + t} + \sum_{t = 1}^{F_r} \beta^{F_r - t} g_{r, (F_f + F_r)(T - 1) + F_f + t}\right)
\end{equation}

Based on Assumption~\ref{assum:input_dynamic}, we can conclude that $\lim_{T \to \infty} B_{(F_f + F_r)T}$ exists, so we let $X_B = \lim_{T \to \infty} B_{(F_f + F_r)T}$.
We consider the equation above, take the expectation over $t$, when $T \to \infty$, we have:
\begin{equation} \label{eq:X_B}
    X_B = \beta^{F_f + F_r} X_B + (1 - \beta) \left(\sum_{t = 1}^{F_f} \beta^{F_f + F_r - t} \cdot mG + \sum_{t = 1}^{F_r} \beta^{F_r - t} \cdot nG \right)
\end{equation}

Based on the equation above, we have $\lim_{T \to \infty} B_{(F_f + F_r)T} = X_B = \frac{\beta^{F_r}(1 - \beta^{F_f})m + (1 - \beta^{F_r})n}{1 - \beta^{F_f + F_r}} G$.

\textbf{2. Convergence of $\Delta_f$}

When $t \in ((F_f + F_r)(T - 1) + F_f, (F_f + F_r)T]$, $\Delta_f$ is not updated, so $\Delta_{f, (F_f + F_r)T} = \Delta_{f, (F_f + F_r)(T - 1) + F_f}$.
Based on the update rule (\ref{eq:delta_update}), we have the following equation:
\begin{equation} \label{eq:Delta_fT}
    \Delta_{f,(F_f + F_r) T} = \beta^{F_f}\Delta_{f,(F_f+F_r) (T-1)} + (1-\beta) \cdot \sum_{t=1}^{F_f}\beta^{F_f-t} \left(g_{f, (F_f + F_r)(T-1) + t} - \widehat{B}_{(F_f+F_r)(T-1)+t-1}\right)
\end{equation}

When $T\to\infty$, and for $1\leq t\leq F_f$, according to (\ref{eq:X_B}), we have:
\begin{equation} \label{eq:B_Tf}
    \lim_{T \to \infty} B_{(F_f + F_r)(T-1)+t-1} = \beta^{t-1}X_B + (1-\beta)\sum_{k=1}^{t-1}\beta^{t-1-k}\cdot mG = \beta^{t-1}X_B + (1-\beta^{t-1}) mG
\end{equation}

When $t\to\infty$, $\widehat{B}_t = B_t/(1-\beta^t) \to B_t$. Let $X_{\Delta_f} = \lim_{T\to\infty}\Delta_{f,(F_f + F_r) T}$ , based on (\ref{eq:Delta_fT}) and (\ref{eq:B_Tf}), we have:
\begin{equation}
\begin{aligned}
     X_{\Delta_f} &= \beta^{F_f}X_{\Delta_f} + (1-\beta)\cdot \sum_{t=1}^{F_f}\beta^{F_f-t} \left[mG - \beta^{t-1}X_B - (1-\beta^{t-1})mG \right] \\
     &=\beta^{F_f}X_{\Delta_f} + (1-\beta)F_f\beta^{F_f-1}  \left(mG -X_B \right)
\end{aligned}
\end{equation}
Based on the equation above, we have $\lim_{T \to \infty} \Delta_{f,(F_f + F_r)T} = X_{\Delta_f} = \frac{F_f\beta^{F_f-1}(1-\beta)\left(1-\beta^{F_r}\right)(m-n)}{\left(1-\beta^{F_f}\right)\left(1-\beta^{F_f+F_r}\right)}G$

\textbf{3. Convergence of $\Delta_r$}

Similarly to (\ref{eq:Delta_fT}), we have:
\begin{equation} \label{eq:Delta_rT}
    \Delta_{r,(F_f + F_r) T} = \beta^{F_r}\Delta_{r,(F_f + F_r) (T-1)} + (1-\beta) \cdot \sum_{t=1}^{F_r}\beta^{F_r-t} \left(g_{r, (F_f+F_r)(T-1) + F_f + t} - \widehat{B}_{(F_f+F_r)(T-1) + F_f + t-1}\right)
\end{equation}

When $T\to\infty$, and for $1\leq t\leq F_r$, according to (\ref{eq:X_B}), we have:

\begin{equation} \label{eq:B_Tr}
\begin{aligned}
    \lim_{T \to \infty} B_{(F_f+F_r)(T-1)+F_f+t-1} &= \beta^{F_f+t-1}X_B + (1-\beta)\left(\sum_{k = 1}^{F_f} \beta^{F_f + t -1 - k} \cdot mG + \sum_{k = 1}^{t-1} \beta^{t-1-k} \cdot nG\right) \\
    &= \beta^{F_f+t-1}X_B + \beta^{t-1}(1-\beta^{F_f})mG + (1-\beta^{t-1})nG
\end{aligned}    
\end{equation}

When $t \to \infty$, $\widehat{B}_t = B_t / (1 - \beta^t) \to B_t$. Let $X_{\Delta_r} = \lim_{T\to\infty}\Delta_{r,(F_f + F_r) T}$ , based on (\ref{eq:Delta_rT}) and (\ref{eq:B_Tr}), we have:

\begin{equation}
\begin{aligned}
     X_{\Delta_r} &= \beta^{F_r}X_{\Delta_r} + (1-\beta)\cdot \sum_{t=1}^{F_r}\beta^{F_r-t} \left[nG - \beta^{F_f+t-1}X_B - \beta^{t-1}(1-\beta^{F_f})mG - (1-\beta^{t-1})nG \right] \\
     &= \beta^{F_r}X_{\Delta_r} + (1-\beta)F_r\beta^{F_r-1}\left[(\beta^{F_f}-1)mG + nG - \beta^{F_f}X_B\right]
\end{aligned}
\end{equation}
Based on the equation above, we have $\lim_{T \to \infty} \Delta_{r,(F_f+F_r)T} = X_{\Delta_r} = \frac{F_r\beta^{F_r-1}(1-\beta)\left(1-\beta^{F_f}\right)(n-m)}{\left(1-\beta^{F_r}\right)\left(1-\beta^{F_f+F_r}\right)}G$
\end{proof}

%% file: Appendix/experiment.tex
\section{More Experiment Results}\label{sec:more_exp}
\subsection{Comparison of Running Time and Memory Consumption} \label{sec:compare_time_memory}
\begin{table}[h]
    \centering
    \caption{Memory consumption and running time of different methods. The model is Phi1.5-1.3B and the optimizer is AdamW. The total steps of unlearning is 300, the batch size per GPU is $4$. All experiments are implemented on two NVIDIA H20 GPUs.} \label{tab:mem_time}
    \small
    \begin{tabular}{c|c c c c c c}
    \toprule[1.2pt]
       Method & Joint & Alternate & DO & DO 8bit& \textbf{DO+} & \textbf{DO+ 8bit}\\
    \midrule[1pt]
      Memory (GB/GPU) & 33.26 & 33.26 & 36.15& 33.57 & 39.04 & 33.68\\
      Running Time (s)& 870 & 443 & 439 & 464 &475 & 468\\
     
    \bottomrule[1.2pt]
    \end{tabular}
\end{table}
Table~\ref{tab:mem_time} provides an overview of the memory consumption and training efficiency for the various methods evaluated. While DO and DO+ introduce additional memory overhead due to their extra optimizer states, peaking at 39.04 GB/GPU for DO+, their 8-bit implementations, DO 8bit and DO+ 8bit, successfully reduce this footprint to 33.57 GB/GPU and 33.68 GB/GPU, respectively. This demonstrates that the 8-bit variants can achieve near-parity in memory usage with the standard Joint and Alternate baselines, with the slight overhead resulting from maintaining full-precision embedding layers and storing quantization coefficients. Notably, the methods using alternating update scheme offer a nearly $2\times$ speedup over the Joint method. This efficiency gain occurs because the Joint baseline doubles the equivalent batch size by simultaneously processing both forget and retain data in each step.
\subsection{Results on TOFU with DPO+GD and NPO+GD loss functions}\label{sec:tofu_dpo_npo}

As shown in Table \ref{tab:tofu_dpo_npo}, DualOptim+ achieves the best overall performance in most cases, consistent with the observations in Table \ref{tab:tofu_idk_me}. However, compared to IDK+GD and ME+GD, the DPO+GD and NPO+GD configurations fail to achieve either stable forgetting or enhanced model utility. This performance discrepancy can be attributed to the conservative weighting strategy inherent in these methods, which relies heavily on the reference model and limits the optimizer's ability to deviate significantly from the initial state.
\begin{table}[h]
    \centering
    \caption{Performance comparison of different methods on TOFU-finetuned \textbf{Phi 1.5} and \textbf{Llama 2}. \textbf{DPO+GD} (targeted) and \textbf{NPO+GD} (untargeted) are adopted as the loss functions for unlearning. The results include Untargeted Forget Efficacy (UFE), Targeted Forget Efficacy (TFE), Model Utility (MU) and the Overall Performance (OVR) for forget 1\%, 5\% data, and 10\% data. Note that the reported results are the average of the results obtained from $5$ runs with different forget sets.} 
    \label{tab:tofu_dpo_npo}
    \tabcolsep=0.4em
    \small
    \begin{tabular}{c|c|cccc|cccc|cccc}
        \toprule[1.2pt]
        \multirow{3}{*}{\textbf{Loss}} & \multirow{3}{*}{\textbf{Method}} & \multicolumn{12}{c}{\textbf{Phi 1.5}}\\
        ~&~ & \multicolumn{4}{c|}{forget 1\% data} & \multicolumn{4}{c|}{forget 5\% data} & \multicolumn{4}{c}{forget 10\% data} \\
        ~&~ & UFE $\uparrow$ & TFE  $\uparrow$ &  MU $\uparrow$ & OVR $\uparrow$& UFE $\uparrow$ & TFE  $\uparrow$ &  MU $\uparrow$ & OVR $\uparrow$ & UFE $\uparrow$ & TFE  $\uparrow$ &  MU $\uparrow$ & OVR $\uparrow$ \\
        \midrule[1pt]
        \multirow{6}{*}{\rotatebox{90}{DPO+GD}}
         & Joint & 79.43& 23.74 & 32.08 & \cellcolor{gray!30} 41.83 
         &\textbf{77.25} & 39.32 & 33.54 & \cellcolor{gray!30}45.91 
         & \textbf{77.67} & 45.46 & 31.48 & \cellcolor{gray!30}46.52\\
        ~& Alternate &  78.56 & 35.11 & 48.23 & \cellcolor{gray!30}52.53
         & 74.34 & 47.26 & 49.68 & \cellcolor{gray!30}55.24
         & 74.49 & \textbf{57.51} & 49.87 & \cellcolor{gray!30}\ul{57.94}\\
        ~& DO & 81.58 & 55.36 & 46.32& \cellcolor{gray!30}57.39
         & 74.86 & 47.10 & 50.94 & \cellcolor{gray!30}55.96
         &72.97 & 52.61 & \ul{50.28} & \cellcolor{gray!30}56.53\\
        ~& DO 8bit & \textbf{83.77}& \ul{57.56}& \ul{47.73} & \cellcolor{gray!30}\ul{59.20}
         & 75.20 & 41.12 & 51.12 & \cellcolor{gray!30}54.64
         & 72.76 & 40.95 & 50.02 & \cellcolor{gray!30}53.44\\
        ~& \textbf{DO+} & 82.92 & \textbf{61.90} & \textbf{48.05} & \cellcolor{gray!30}\textbf{60.23}
         & \ul{76.04} & \textbf{55.56} & \ul{51.33} & \cellcolor{gray!30}\textbf{58.57}
         & \ul{74.67} & \ul{56.46} & \textbf{50.78} & \cellcolor{gray!30}\textbf{58.18}\\
        ~& \textbf{DO+ 8bit} & \ul{83.46} & 57.44 & 47.61 & \cellcolor{gray!30}59.03 
         & 75.96& \ul{50.13} & \textbf{51.40} & \cellcolor{gray!30}\ul{57.23}
         & 73.14 & 46.21& 50.24 & \cellcolor{gray!30}54.96\\
        \midrule[1pt]
        \multirow{6}{*}{\rotatebox{90}{NPO+GD}} 
         & Joint & \textbf{78.87} & -- & 29.89 & \cellcolor{gray!30}54.38
         & \textbf{74.68} & -- & 31.71 & \cellcolor{gray!30}53.20
         & \textbf{73.73} & -- & 27.78 & \cellcolor{gray!30}50.76\\
        ~& Alternate & 74.62 & -- & 48.32 & \cellcolor{gray!30}61.47
         & 70.46 & -- & 50.17& \cellcolor{gray!30}60.32 
         & 67.81 & -- & 49.15 & \cellcolor{gray!30}58.48\\
        ~& DO & 74.88 & --& 48.47  & \cellcolor{gray!30}61.68
         & 70.60 & --& 49.89 & \cellcolor{gray!30}60.25
         & 67.33 & -- & 48.73 & \cellcolor{gray!30}58.03\\
        ~& DO 8bit & 75.51 & -- & \ul{48.77} & \cellcolor{gray!30}62.14
         & 70.17 & --& 50.27 & \cellcolor{gray!30}60.23
         & 66.88 & -- & \ul{50.07} & \cellcolor{gray!30}58.48\\
        ~& \textbf{DO+} & 75.80& -- & 48.72 & \cellcolor{gray!30}\ul{62.26}
         & \ul{71.63} & -- & \textbf{51.52}& \cellcolor{gray!30}\textbf{61.58}
         & \ul{69.12} & -- & 49.71 & \cellcolor{gray!30}\textbf{59.42}\\
        ~& \textbf{DO+ 8bit} & \ul{76.10} & -- & \textbf{48.79} & \cellcolor{gray!30}\textbf{62.45}
         & 70.75 & --& \ul{50.63} & \cellcolor{gray!30}\ul{60.69}
         & 67.11 & -- & \textbf{51.25} & \cellcolor{gray!30}\ul{59.18}\\
        \bottomrule[1.2pt]
        \toprule[1.2pt]
        \multirow{3}{*}{\textbf{Loss}} & \multirow{3}{*}{\textbf{Method}} & \multicolumn{12}{c}{\textbf{Llama 2}}\\
        ~&~ & \multicolumn{4}{c|}{forget 1\% data} & \multicolumn{4}{c|}{forget 5\% data} & \multicolumn{4}{c}{forget 10\% data} \\
        ~&~ & UFE $\uparrow$ & TFE  $\uparrow$ &  MU $\uparrow$ & OVR $\uparrow$& UFE $\uparrow$ & TFE  $\uparrow$ &  MU $\uparrow$ & OVR $\uparrow$ & UFE $\uparrow$ & TFE  $\uparrow$ &  MU $\uparrow$ & OVR $\uparrow$ \\
        \midrule[1pt]
        \multirow{6}{*}{\rotatebox{90}{DPO+GD}}
         & Joint & \textbf{92.02}& 45.74 & 44.19 & \cellcolor{gray!30}56.54
         & \textbf{88.11} & 45.17 & 66.98 & \cellcolor{gray!30}66.81
         & \textbf{85.14} & \textbf{49.94} & 65.05 & \cellcolor{gray!30}66.30\\
        ~& Alternate & \ul{88.27} &\textbf{53.77} & 73.67 & \cellcolor{gray!30}\textbf{72.35}
        &83.13 & 46.41 & 73.46 & \cellcolor{gray!30}69.11
        & 80.91 & 36.59 & 74.03 & \cellcolor{gray!30}66.39\\
        ~& DO & 89.26 & \ul{48.24} & 72.33 & \cellcolor{gray!30}\ul{70.54}
        &83.67 & 45.31 & 71.42& \cellcolor{gray!30}67.96
        & 80.77 & 38.55 & 72.43 & \cellcolor{gray!30}66.04\\
        ~& DO 8bit & 89.47 &46.54 & 72.32 & \cellcolor{gray!30}70.16
        & \ul{84.49} &\textbf{48.19} & 73.07 & \cellcolor{gray!30}\ul{69.71}
        & \ul{82.18} & 40.49 & 73.31 & \cellcolor{gray!30}67.32\\
        ~& \textbf{DO+} & 79.49 & 45.96 & \textbf{77.51} & \cellcolor{gray!30}70.12
        & 77.99 & 41.86 & \textbf{76.51} & \cellcolor{gray!30}68.22
        & 78.43 & \ul{47.55} & \textbf{77.19} & \cellcolor{gray!30}\textbf{70.09}\\
        ~& \textbf{DO+ 8bit} &82.04 & 43.57 & \ul{77.01} & \cellcolor{gray!30}69.91
        & 81.94 & \ul{47.83} &\ul{75.06} & \cellcolor{gray!30}\textbf{69.98}
        & 80.15 & 45.71 & \ul{75.10} & \cellcolor{gray!30}\ul{69.17}\\
        \midrule[1pt]
        \multirow{6}{*}{\rotatebox{90}{NPO+GD}} 
         & Joint & \ul{77.31} & -- & 55.56 & \cellcolor{gray!30}66.44
         & 69.67 & -- & 63.12 & \cellcolor{gray!30}66.40
         & 67.90 & --& 64.46 & \cellcolor{gray!30}66.18\\
        ~& Alternate & 73.51 & -- & 74.40& \cellcolor{gray!30}73.96
        & 75.32 & -- & \ul{75.69} & \cellcolor{gray!30}75.51
        & \ul{72.77} & -- & \ul{75.90} & \cellcolor{gray!30}\ul{74.34}\\
        ~& DO & 73.83 & -- & 74.51 & \cellcolor{gray!30}74.17
        & 74.31 & -- &74.52 & \cellcolor{gray!30}74.42
        & 71.78& --& 73.99 & \cellcolor{gray!30}72.89\\
        ~& DO 8bit & 73.85 & -- & 74.22 & \cellcolor{gray!30}74.04
        & 75.35 & -- & 75.11 & \cellcolor{gray!30}75.24
        & 72.42 & -- & 75.75 & \cellcolor{gray!30}74.09\\
        ~& \textbf{DO+} & \textbf{78.51} & -- & \textbf{75.07} & \cellcolor{gray!30}\textbf{76.80}
        & \textbf{75.94} & -- & 75.37 & \cellcolor{gray!30}\ul{75.65}
        & 71.95 & -- & 73.95 & \cellcolor{gray!30}72.95\\
        ~& \textbf{DO+ 8bit} & 76.58 & -- &  \ul{74.91} & \cellcolor{gray!30}\ul{75.75}
        & \ul{75.42} & -- & \textbf{76.01} & \cellcolor{gray!30}\textbf{75.72}
            & \textbf{73.79} & -- & \textbf{76.03} & \cellcolor{gray!30}\textbf{74.92}\\
        \bottomrule[1.2pt]
    \end{tabular}
    \vspace{-0.5em}
\end{table}

\subsection{Results on TOFU with LoRA} \label{sec:tofu_lora}
To evaluate the effectiveness of our method within a parameter-efficient fine-tuning framework, we apply LoRA \citep{hu2022lora} with a rank of 8 to Llama 2. As shown in Table \ref{tab:tofu_idk_me_lora}, the performance gains from both DualOptim and DualOptim+ are less pronounced than in full-parameter unlearning, but DualOptim+ and its 8bit variant exhibit the best performance in most cases. Furthermore, the results indicate that the performance gap between LoRA and full-parameter tuning widens as the volume of data to be forgotten increases. This is because the limited degrees of freedom in a low-rank subspace are insufficient to encode the complex modifications required for larger-scale unlearning tasks . Conversely, LoRA significantly mitigates the excessive forgetting on the retain set that is frequently observed when using the Joint updating scheme.
\begin{table}[H]
    \centering
    \caption{Performance comparison of different methods on TOFU-finetuned \textbf{Llama 2} with \textbf{LoRA} ($\mathrm{rank=8}$). \textbf{IDK+GD} (targeted) and \textbf{ME+GD} (untargeted) are adopted as the loss functions for unlearning. The results include Untargeted Forget Efficacy (UFE), Targeted Forget Efficacy (TFE), Model Utility (MU) and the Overall Performance (OVR) for forget 1\%, 5\% data, and 10\% data. Note that the reported results are the average of the results obtained from $5$ runs with different forget sets.} 
    \label{tab:tofu_idk_me_lora}
    \tabcolsep=0.4em
    \small
    \begin{tabular}{c|c|cccc|cccc|cccc}
        \toprule[1.2pt]
        \multirow{3}{*}{\textbf{Loss}} & \multirow{3}{*}{\textbf{Method}} & \multicolumn{12}{c}{\textbf{Llama 2 + LoRA}}\\
        ~&~ & \multicolumn{4}{c|}{forget 1\% data} & \multicolumn{4}{c|}{forget 5\% data} & \multicolumn{4}{c}{forget 10\% data} \\
        ~&~ & UFE $\uparrow$ & TFE  $\uparrow$ &  MU $\uparrow$ & OVR $\uparrow$& UFE $\uparrow$ & TFE  $\uparrow$ &  MU $\uparrow$ & OVR $\uparrow$ & UFE $\uparrow$ & TFE  $\uparrow$ &  MU $\uparrow$ & OVR $\uparrow$ \\
        \midrule[1pt]
        \multirow{6}{*}{\rotatebox{90}{IDK+GD}}
         & Joint & \textbf{73.70} & 59.50 & 73.07 & \cellcolor{gray!30}69.84
         & \textbf{70.10} & 68.60 & 68.82 & \cellcolor{gray!30}69.08
         & \textbf{64.72} & \textbf{67.54} & 64.92 & \cellcolor{gray!30}65.52\\
        ~& Alternate & 71.69 &\ul{66.40} & \textbf{74.57} & \cellcolor{gray!30}\ul{71.81}
        &64.58 & 66.63 & 70.51 & \cellcolor{gray!30}68.06
        & 59.11 & 56.41 & 66.14 & \cellcolor{gray!30}61.95\\
        ~& DO & 72.01 & \ul{66.40} & 74.41 & \cellcolor{gray!30}\ul{71.81}
        &68.65 & 69.76 & 70.56 & \cellcolor{gray!30}69.88
        & \ul{64.17} & 62.90 & 64.42 & \cellcolor{gray!30}64.42\\
        ~& DO 8bit & 72.03 &\ul{66.40} & \ul{74.51} & \cellcolor{gray!30}\textbf{71.86}
        & 68.77 & 69.77 & 70.63 & \cellcolor{gray!30}69.95
        & 64.08 & \ul{63.13} & 64.62 & \cellcolor{gray!30}64.11\\
        ~& \textbf{DO+} & \ul{72.09} & \textbf{66.41} & 73.87 & \cellcolor{gray!30}71.56
        & 68.77 & \ul{70.10} & \ul{73.05} & \cellcolor{gray!30}\ul{71.24}
        & 61.22 & 62.76 & \ul{69.88} & \cellcolor{gray!30}\ul{65.94}\\
        ~& \textbf{DO+ 8bit} &71.76 & \textbf{66.41} & 73.91 & \cellcolor{gray!30}71.50
        &\ul{68.86} & \textbf{70.19} &\textbf{73.11} & \cellcolor{gray!30}\textbf{71.32}
        & 61.26 & 62.57 & \textbf{69.97} & \cellcolor{gray!30}\textbf{65.95}\\
        \midrule[1pt]
        \multirow{6}{*}{\rotatebox{90}{ME+GD}} 
         & Joint & \textbf{96.14} & -- & 73.92 & \cellcolor{gray!30}85.03
         & \ul{95.01} & -- & 75.04 & \cellcolor{gray!30}85.02
         & 93.27 & --& 75.79 & \cellcolor{gray!30}84.53\\
        ~& Alternate & 95.78 & -- & 75.97 & \cellcolor{gray!30}85.88
        & 92.69 & -- & 75.56 & \cellcolor{gray!30}84.13
        & 88.79 & -- & 74.62 & \cellcolor{gray!30}81.71\\
        ~& DO & 95.87 & -- & 76.00 & \cellcolor{gray!30}85.94
        & 94.97 & -- &76.00 & \cellcolor{gray!30}85.49
        & 93.63& --& 75.57 & \cellcolor{gray!30}84.60\\
        ~& DO 8bit & \ul{95.91} & -- & 76.09 & \cellcolor{gray!30}\ul{86.00}
        & \textbf{95.11} & -- & 76.13 & \cellcolor{gray!30}\textbf{85.62}
        & \textbf{93.84} & -- & 75.68 & \cellcolor{gray!30}84.76\\
        ~& \textbf{DO+} & 95.61 & -- & \textbf{76.46} & \cellcolor{gray!30}\textbf{86.04}
        & 94.89 & -- & \textbf{76.28} & \cellcolor{gray!30}\ul{85.59}
        & \ul{93.73} & -- & \ul{76.04} & \cellcolor{gray!30}\textbf{84.89}\\
        ~& \textbf{DO+ 8bit} & 95.04 & -- &  \ul{76.40}& \cellcolor{gray!30}85.73
        & 94.90 & -- & \ul{76.22} & \cellcolor{gray!30}85.56
        & 93.70 & -- & \textbf{76.05} & \cellcolor{gray!30}\ul{84.88}\\
        \bottomrule[1.2pt]
    \end{tabular}
\end{table}

\subsection{Unlearning Metrics and Losses over Time Steps} \label{sec:curve_metric}

As depicted in Figure \ref{fig:curve} (a) and (b), DualOptim+ exhibits the most effective and stable performance among the evaluated methods. While the Joint updating scheme achieves similar forget efficacy, its Model Utility is substantially lower than other methods. As observed in Figure Figure \ref{fig:curve} (c) and (d), both forgetting and retaining losses converge rapidly under the Joint updating scheme, suggesting the model becomes trapped in a suboptimal local minimum. In contrast, methods employing the alternate updating scheme, i.e., Alternate, DualOptim, and DualOptim+, converge more gradually, enhancing model's exploration capability to find superior local minima.
\begin{figure}[H]
    \centering
    \subfigure[Forget Efficacy]{
    \includegraphics[width=0.235\textwidth]{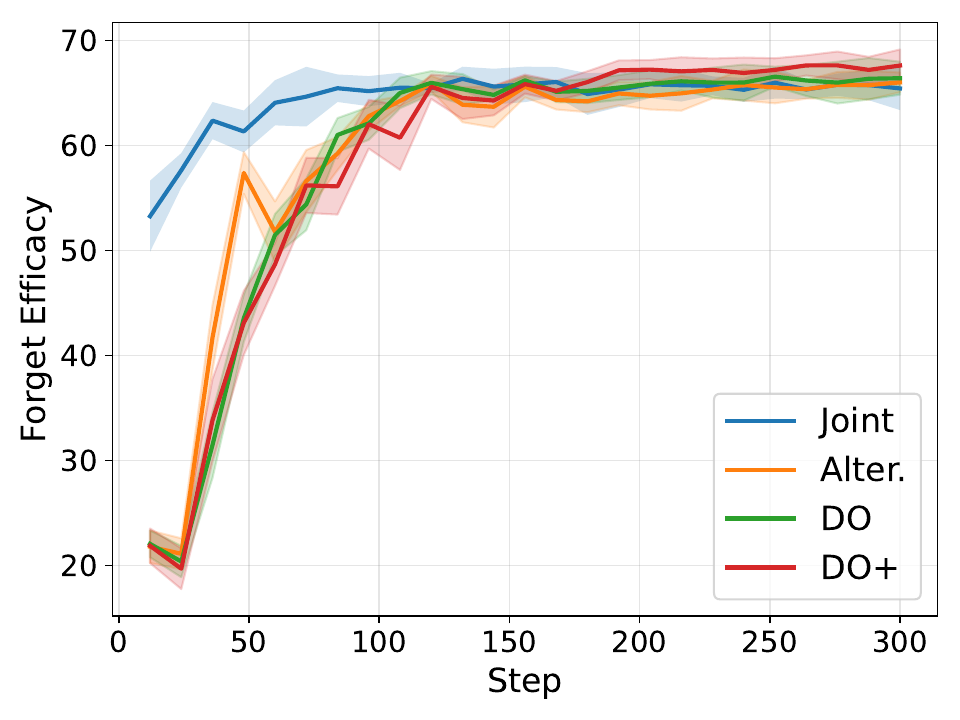}
    }
    \subfigure[Model Utility]{
    \includegraphics[width=0.235\textwidth]{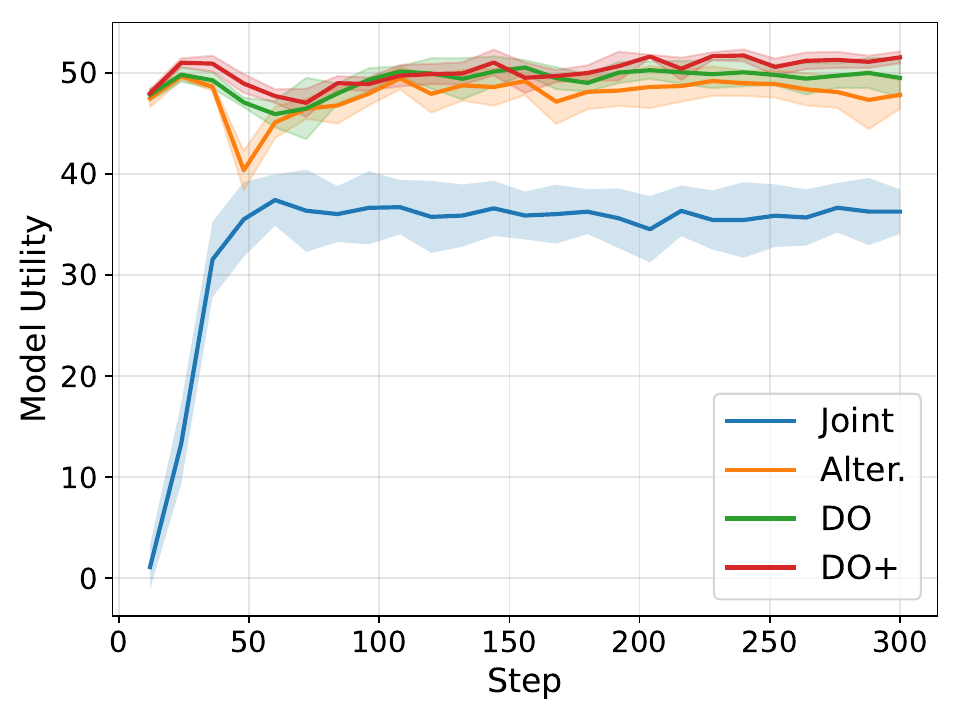}
    }
    \subfigure[Forgetting Loss]{
    \includegraphics[width=0.235\textwidth]{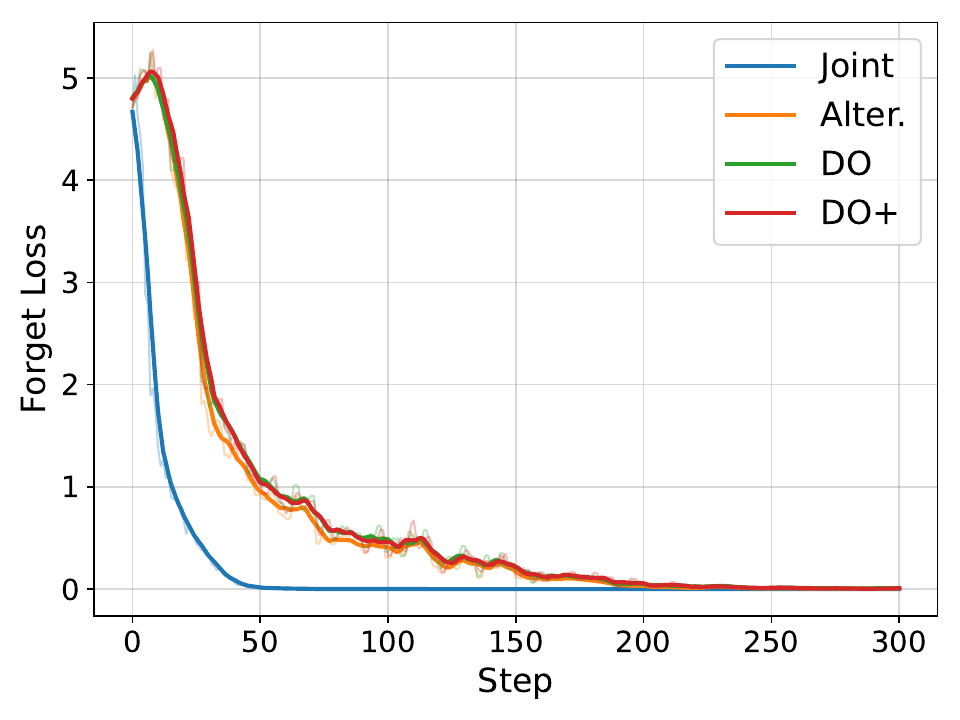}
    }
    \subfigure[Retaining Loss]{
    \includegraphics[width=0.235\textwidth]{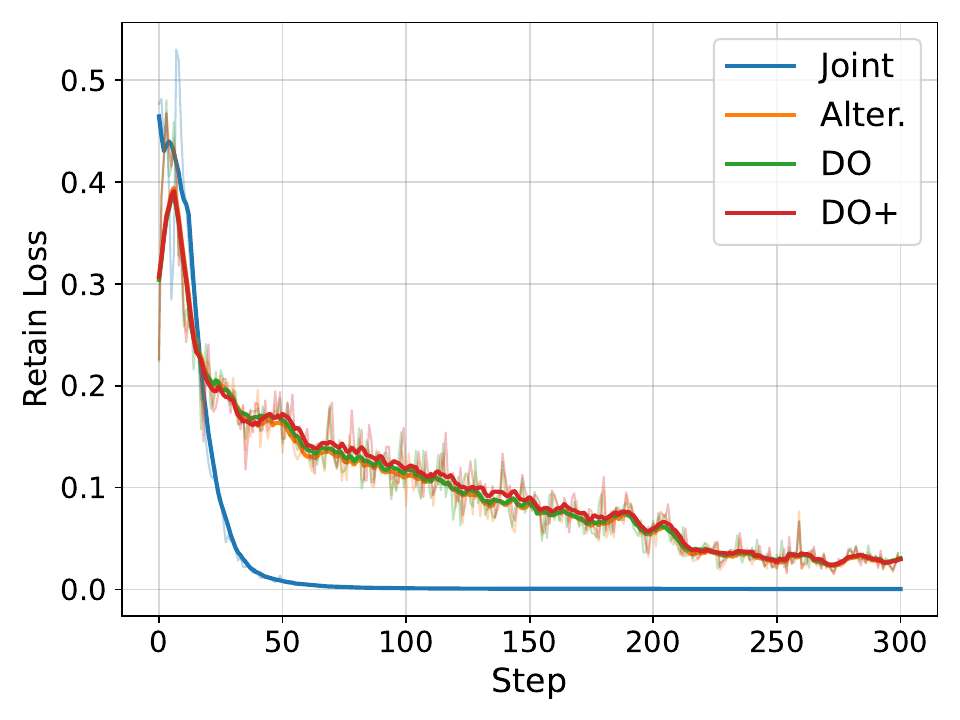}
    }
    \vspace{-0.5em}
    \caption{Comparison of unlearning metrics and losses over time steps. We plot the \textbf{(a)} forget efficacy, i.e., the average of targeted forget efficacy and untargeted forget efficacy, \textbf{(b)} model utility, \textbf{(c)} forgetting loss, and \textbf{(d)} retaining loss. Note that the results are collected when forgetting $5\%$ data of TOFU using IDK+GD loss function. The model is TOFU-finetuned Phi 1.5.} 
    \label{fig:curve}
    \vspace{-0.5em}
\end{figure}


\section{Comparison with Federated Learning Methods} \label{sec:compare_fl}
While federated learning \citep{konecny2015federated} focuses on overcoming client drift caused by non-IID data across a shared objective, machine unlearning must navigate adversarial objectives: the mandate to forget specific information while simultaneously retaining general knowledge. These loss functions are often diametrically opposed, leading to catastrophic forgetting, which is a phenomenon that is not a central concern in standard FL but is the primary bottleneck in unlearning. Our work specifically addresses this tension through a novel optimizer design.

Furthermore, we compare the proposed DO+ with some specific FL methods, i.e., SCAFFOLD \citep{pmlr-v119-karimireddy20a}, MIME \citep{karimireddy2021mime}, FedCM \citep{fedcm}, and Local Adam \citep{local_adam}. While FL involves optimizing multiple local models, unlearning tasks require balancing multiple objectives. Therefore, we have to adapt the core design of these FL methods mentioned to the unlearning context to ensure compatibility with our optimization framework. In the following, we list their implementations and highlight the differences between these adapted methods and DO+ regarding base and delta state updates.
\begin{itemize}
    \item \textbf{SCAFFOLD, MIME:} These two methods share the similar core design. In the context of unlearning tasks, their updating rule of delta state $\Delta$ will be $\Delta \leftarrow \beta\Delta + (1-\beta)(g-B)$, which is the same as DO+. They update the base state $B$ once only after a full forget-retain period (i.e., $F_f+F_r$ batches of data) by the rule $B\leftarrow B+\frac{1}{2}(\widehat{\Delta}_f + \widehat{\Delta}_r)$. 
    \item \textbf{FedCM:} The updating rule of delta state is $\Delta \leftarrow \beta\Delta + (1-\beta)g$. Same as SCAFFOLD and MIME, the base state is updated once only after a full forget-retain period by $B\leftarrow B+\frac{1}{2}(\widehat{\Delta}_f + \widehat{\Delta}_r)$.
    \item \textbf{Local Adam:} It only maintains two states for forgetting and retaining objectives, which is similar to DO. The difference is that the forget state $S_f$ and retain state $S_r$ will be merged as $S=\frac{1}{2}(S_f+S_r)$ after a full forget-retain period.
    \item \textbf{DO+:} The updating rule of delta state is $\Delta \leftarrow \beta\Delta + (1-\beta)(g-\widehat{B})$. The base state is updated after each data batch by $B \leftarrow \beta B + (1-\beta)g$.
\end{itemize}

Additionally, we compare these methods with DO+ under the setting of forgetting 5\% data of TOFU, IDK+GD loss, Phi-1.5. As shown in the Table \ref{tab:fl}, our method achieves the best performance, further underscoring the effectiveness of DO+ in LLM unlearning. To some extent, it can be attributed to the difference that DO+ updates the base state more frequently.

\begin{table}[H]
    \centering
    \caption{Unlearning performance of DO+ and different federated learning methods. The experiment setting is forgetting 5\% data of TOFU, IDK+GD loss, Phi-1.5.}
    \small
    \begin{tabular}{c|c c c c}
    \toprule[1.2pt]
       ~ & UFE & TFE & MU & AVG\\
    \midrule[1pt]
    SCAFFOLD & 67.24 & 66.89 & 50.44 & \cellcolor{gray!30}58.76\\
    FedCM & 70.35 & 65.32 & 49.34 & \cellcolor{gray!30}58.59\\
    Local Adam & 70.43 & 65.90 & 49.68 & \cellcolor{gray!30}58.93\\
    \textbf{DO+} & 67.63 & 67.60& 51.52 & \cellcolor{gray!30}\textbf{59.57} \\
    \bottomrule[1.2pt]
    \end{tabular}
    \vspace{-1em}
    \label{tab:fl}
\end{table}

%% file: Appendix/imple.tex
\section{Implementation Details}\label{sec:imple}
All experiments are conducted on NVIDIA H20 GPUs, and we utilize PyTorch FSDP to reduce memory costs. We employ DualOptim+ and baselines based on AdamW optimizer with a weight decay of $0.01$, betas of ($0.9$, $0.95$). A linear warm-up learning rate in the first epoch and a linearly decaying learning rate in the subsequent epochs are used.

\textbf{Unlearning in Fictitious Scenario.} For the experments on the TOFU dataset, we use fine-tuned Phi1.5-1.3B and Llama2-chat-7B models released by the original paper \citep{mainitofu} as the target models. For Phi 1.5, we use two NVIDIA H20 GPUs, and the effective batch size is set $40$. For Llama 2, we use eight NVIDIA H20 GPUs, and the effective batch size is set $128$. The initial learning rate is set to $1\times 10^{-5}$. The total unlearn steps $N$ is $300$. The forget frequency $F_f$ is $1$. The retain frequency $F_r$ is $5$. Following the setup in \citet{yuan2024closer}, the parameter $\alpha$ in ME+GD is set to $0.1$. The $\beta$ in DPO and NPO is set to $0.1$. The evaluation metrics are calculated using the codes in the official repository\footnote{\href{https://github.com/sail-sg/closer-look-LLM-unlearning}{https://github.com/sail-sg/closer-look-LLM-unlearning}}.

\textbf{Unlearning in Real-world Scenario.} Follow \citet{liu2024learning,yuan2024closer}, we use Llama-3-8B-Instruct as the target model. We use eight NVIDIA H20 GPUs, and the effective batch size is set $128$. The initial learning rate is set to $5\times 10^{-6}$. We use the repository\footnote{\href{https://github.com/EleutherAI/lm-evaluation-harness}{https://github.com/EleutherAI/lm-evaluation-harness}} of lm-evaluation-harness \citep{eval-harness} to evaluate downstream tasks with default configurations. Other configurations are consistent with TOFU dataset. 

\textbf{Safety Alignment.} We use Llama-3-8B-Instruct fine-tuend on Alpaca\footnote{\href{https://huggingface.co/PKU-Alignment/alpaca-8b-reproduced-llama-3}{https://huggingface.co/PKU-Alignment/alpaca-8b-reproduced-llama-3}} as the target model. We use eight NVIDIA H20 GPUs, and the effective batch size is set $128$. The initial learning rate is set to $1\times 10^{-5}$. The total epochs are $3$.
Following the settings in \citet{bianchisafety}, for \textbf{Joint} method, we calculate cross-entropy loss on batches randomly sampled from a mixture of 20,000 Alpaca instructions and 2,000 safety instructions. 
For \textbf{methods using alternating update scheme}, the forget frequency $F_f$ and retain frequency $F_r$ are set to $1$ and $10$, respectively, to simulate the ratio of safety instructions to general instructions, i.e., $1:10$.

\section{Discussion on Evaluation Metrics}\label{sec:example}
As shown in Table \ref{tab:exmaple}, the first two examples achieve high Untargeted Forget Efficacy (UFE) despite consisting of nonsensical strings that fail to provide a coherent rejection. Furthermore, coherent rejection responses (the last two examples) typically yield lower UFE scores than nonsensical outputs. These observations suggest that UFE alone is an insufficient metric for evaluating targeted unlearning. In contrast, Targeted Forget Efficacy (TFE) effectively distinguishes between appropriate rejections and nonsense. Consequently, for targeted unlearning tasks, we employ the mean of UFE and TFE as the comprehensive forget efficacy metric.

\begin{table}[h]
    \centering
    \caption{The outputs of unlearned Llama2 on the forget set of TOFU and the corresponding Untargeted Forget Efficacy (UFE) and Targeted Forget Efficacy (TFE). The first two examples are generated by the model unlearned using ME+GD loss function. The last two examples are generated using IDK+GD loss function.}\label{tab:exmaple}
    \begin{tabular}{c|cc}
    \toprule[1.2pt]
       Generated Text  & UFE & TFE \\
    \midrule[1pt]
        fullL fullL fullL fullL fullL... & 94.87 & 0.00\\
        ". born born". born born within... & 93.98 & 0.00\\
     \midrule[0.75pt]
     That's a blind spot in my knowledge. & 69.84 & 100.00\\
        I must admit, I don't know. & 69.39 & 96.90\\   
    \bottomrule[1.2pt]
    \end{tabular}
    
\end{table}

%% file: main.bib
@inproceedings{bourtoule2021machine,
  title={Machine unlearning},
  author={Bourtoule, Lucas and Chandrasekaran, Varun and Choquette-Choo, Christopher A and Jia, Hengrui and Travers, Adelin and Zhang, Baiwu and Lie, David and Papernot, Nicolas},
  booktitle={2021 IEEE symposium on security and privacy (SP)},
  pages={141--159},
  year={2021},
  organization={IEEE}
}

@inproceedings{
zhong2025dualoptim,
title={DualOptim: Enhancing Efficacy and Stability in Machine Unlearning with Dual Optimizers},
author={Xuyang Zhong and Haochen Luo and Chen Liu},
booktitle={The Thirty-ninth Annual Conference on Neural Information Processing Systems},
year={2025},
url={https://openreview.net/forum?id=77zz0JTNjn}
}

@inproceedings{fan2024salun,
  title={SalUn: Empowering Machine Unlearning via Gradient-Based Weight Saliency in Both Image Classification and Generation},
  author={Fan, Chongyu and Liu, Jiancheng and Zhang, Yihua and Wei, Dennis and Wong, Eric and Liu, Sijia},
  booktitle={International Conference on Learning Representations},
  year={2024}
}

@article{huang2025unified,
  title={Unified gradient-based machine unlearning with remain geometry enhancement},
  author={Huang, Zhehao and Cheng, Xinwen and Zheng, JingHao and Wang, Haoran and He, Zhengbao and Li, Tao and Huang, Xiaolin},
  journal={Advances in Neural Information Processing Systems},
  volume={37},
  pages={26377--26414},
  year={2024}
}

@inproceedings{yuan2024closer,
  title={A Closer Look at Machine Unlearning for Large Language Models},
  author={Yuan, Xiaojian and Pang, Tianyu and Du, Chao and Chen, Kejiang and Zhang, Weiming and Lin, Min},
  booktitle={International Conference on Learning Representations},
  year={2025}
}

@inproceedings{
zhang2024negative,
title={Negative Preference Optimization: From Catastrophic Collapse to Effective Unlearning},
author={Ruiqi Zhang and Licong Lin and Yu Bai and Song Mei},
booktitle={First Conference on Language Modeling},
year={2024},
url={https://openreview.net/forum?id=MXLBXjQkmb}
}

@misc{jordan2024muon,
  author       = {Keller Jordan and Yuchen Jin and Vlado Boza and Jiacheng You and
                  Franz Cesista and Laker Newhouse and Jeremy Bernstein},
  title        = {Muon: An optimizer for hidden layers in neural networks},
  year         = {2024},
  url          = {https://kellerjordan.github.io/posts/muon/}
}

@article{li2023textbooks,
  title={Textbooks are all you need ii: phi-1.5 technical report},
  author={Li, Yuanzhi and Bubeck, S{\'e}bastien and Eldan, Ronen and Del Giorno, Allie and Gunasekar, Suriya and Lee, Yin Tat},
  journal={arXiv preprint arXiv:2309.05463},
  year={2023}
}

@article{touvron2023llama,
  title={Llama 2: Open foundation and fine-tuned chat models},
  author={Touvron, Hugo and Martin, Louis and Stone, Kevin and Albert, Peter and Almahairi, Amjad and Babaei, Yasmine and Bashlykov, Nikolay and Batra, Soumya and Bhargava, Prajjwal and Bhosale, Shruti and others},
  journal={arXiv preprint arXiv:2307.09288},
  year={2023}
}

@article{grattafiori2024llama,
  title={The llama 3 herd of models},
  author={Grattafiori, Aaron and Dubey, Abhimanyu and Jauhri, Abhinav and Pandey, Abhinav and Kadian, Abhishek and Al-Dahle, Ahmad and Letman, Aiesha and Mathur, Akhil and Schelten, Alan and Vaughan, Alex and others},
  journal={arXiv preprint arXiv:2407.21783},
  year={2024}
}

@inproceedings{mainitofu,
  title={TOFU: A Task of Fictitious Unlearning for LLMs},
  author={Maini, Pratyush and Feng, Zhili and Schwarzschild, Avi and Lipton, Zachary Chase and Kolter, J Zico},
  booktitle={First Conference on Language Modeling},
 year={2024}
}

@inproceedings{bianchisafety,
  title={Safety-Tuned LLaMAs: Lessons From Improving the Safety of Large Language Models that Follow Instructions},
  author={Bianchi, Federico and Suzgun, Mirac and Attanasio, Giuseppe and Rottger, Paul and Jurafsky, Dan and Hashimoto, Tatsunori and Zou, James},
  booktitle={The Twelfth International Conference on Learning Representations},
year={2024}
}

@misc{alpaca,
  author = {Rohan Taori and Ishaan Gulrajani and Tianyi Zhang and Yann Dubois and Xuechen Li and Carlos Guestrin and Percy Liang and Tatsunori B. Hashimoto },
  title = {Stanford Alpaca: An Instruction-following LLaMA model},
  year = {2023},
  publisher = {GitHub},
  journal = {GitHub repository},
  howpublished = {\url{https://github.com/tatsu-lab/stanford_alpaca}},
}

@inproceedings{
rafailov2023direct,
title={Direct Preference Optimization: Your Language Model is Secretly a Reward Model},
author={Rafael Rafailov and Archit Sharma and Eric Mitchell and Christopher D Manning and Stefano Ermon and Chelsea Finn},
booktitle={Thirty-seventh Conference on Neural Information Processing Systems},
year={2023},
url={https://openreview.net/forum?id=HPuSIXJaa9}
}

@inproceedings{
hu2022lora,
title={Lo{RA}: Low-Rank Adaptation of Large Language Models},
author={Edward J Hu and yelong shen and Phillip Wallis and Zeyuan Allen-Zhu and Yuanzhi Li and Shean Wang and Lu Wang and Weizhu Chen},
booktitle={International Conference on Learning Representations},
year={2022},
url={https://openreview.net/forum?id=nZeVKeeFYf9}
}

@article{liu2024learning,
  title={Learning to refuse: Towards mitigating privacy risks in llms},
  author={Liu, Zhenhua and Zhu, Tong and Tan, Chuanyuan and Chen, Wenliang},
  journal={arXiv preprint arXiv:2407.10058},
  year={2024}
}

@inproceedings{
hendrycks2021measuring,
title={Measuring Massive Multitask Language Understanding},
author={Dan Hendrycks and Collin Burns and Steven Basart and Andy Zou and Mantas Mazeika and Dawn Song and Jacob Steinhardt},
booktitle={International Conference on Learning Representations},
year={2021},
url={https://openreview.net/forum?id=d7KBjmI3GmQ}
}

@article{Clark2018ThinkYH,
  title={Think you have Solved Question Answering? Try ARC, the AI2 Reasoning Challenge},
  author={Peter Clark and Isaac Cowhey and Oren Etzioni and Tushar Khot and Ashish Sabharwal and Carissa Schoenick and Oyvind Tafjord},
  journal={ArXiv},
  year={2018},
  volume={abs/1803.05457},
  url={https://api.semanticscholar.org/CorpusID:3922816}
}

@article{Cobbe2021TrainingVT,
  title={Training Verifiers to Solve Math Word Problems},
  author={Karl Cobbe and Vineet Kosaraju and Mo Bavarian and Mark Chen and Heewoo Jun and Lukasz Kaiser and Matthias Plappert and Jerry Tworek and Jacob Hilton and Reiichiro Nakano and Christopher Hesse and John Schulman},
  journal={ArXiv},
  year={2021},
  volume={abs/2110.14168},
  url={https://api.semanticscholar.org/CorpusID:239998651}
}

@inproceedings{joshi-etal-2017-triviaqa,
    title = "{T}rivia{QA}: A Large Scale Distantly Supervised Challenge Dataset for Reading Comprehension",
    author = "Joshi, Mandar  and
      Choi, Eunsol  and
      Weld, Daniel  and
      Zettlemoyer, Luke",
    editor = "Barzilay, Regina  and
      Kan, Min-Yen",
    booktitle = "Proceedings of the 55th Annual Meeting of the Association for Computational Linguistics (Volume 1: Long Papers)",
    month = jul,
    year = "2017",
    address = "Vancouver, Canada",
    publisher = "Association for Computational Linguistics",
    url = "https://aclanthology.org/P17-1147/",
    doi = "10.18653/v1/P17-1147",
    pages = "1601--1611",
}

@inproceedings{Lin2021TruthfulQAMH,
  title={TruthfulQA: Measuring How Models Mimic Human Falsehoods},
  author={Stephanie C. Lin and Jacob Hilton and Owain Evans},
  booktitle={Annual Meeting of the Association for Computational Linguistics},
  year={2021},
  url={https://api.semanticscholar.org/CorpusID:237532606}
}

@article{Bai2022TrainingAH,
  title={Training a Helpful and Harmless Assistant with Reinforcement Learning from Human Feedback},
  author={Yuntao Bai and Andy Jones and Kamal Ndousse and Amanda Askell and Anna Chen and Nova Dassarma and Dawn Drain and Stanislav Fort and Deep Ganguli and T. J. Henighan and Nicholas Joseph and Saurav Kadavath and John Kernion and Tom Conerly and Sheer El-Showk and Nelson Elhage and Zac Hatfield-Dodds and Danny Hernandez and Tristan Hume and Scott Johnston and Shauna Kravec and Liane Lovitt and Neel Nanda and Catherine Olsson and Dario Amodei and Tom B. Brown and Jack Clark and Sam McCandlish and Chris Olah and Benjamin Mann and Jared Kaplan},
  journal={ArXiv},
  year={2022},
  volume={abs/2204.05862},
  url={https://api.semanticscholar.org/CorpusID:248118878}
}

@inproceedings{rottger-etal-2024-xstest,
    title = "{XST}est: A Test Suite for Identifying Exaggerated Safety Behaviours in Large Language Models",
    author = {R{\"o}ttger, Paul  and
      Kirk, Hannah  and
      Vidgen, Bertie  and
      Attanasio, Giuseppe  and
      Bianchi, Federico  and
      Hovy, Dirk},
    editor = "Duh, Kevin  and
      Gomez, Helena  and
      Bethard, Steven",
    booktitle = "Proceedings of the 2024 Conference of the North American Chapter of the Association for Computational Linguistics: Human Language Technologies (Volume 1: Long Papers)",
    month = jun,
    year = "2024",
    address = "Mexico City, Mexico",
    publisher = "Association for Computational Linguistics",
    url = "https://aclanthology.org/2024.naacl-long.301/",
    doi = "10.18653/v1/2024.naacl-long.301",
    pages = "5377--5400",
}

@misc{metallamaguard2,
  author =       {Llama Team},
  title =        {Meta Llama Guard 2},
  howpublished = {\url{https://github.com/meta-llama/PurpleLlama/blob/main/Llama-Guard2/MODEL_CARD.md}},
  year =         {2024}
}

@inproceedings{fanton-etal-2021-human,
    title = "Human-in-the-Loop for Data Collection: a Multi-Target Counter Narrative Dataset to Fight Online Hate Speech",
    author = "Fanton, Margherita  and
      Bonaldi, Helena  and
      Tekiro{\u{g}}lu, Serra Sinem  and
      Guerini, Marco",
    editor = "Zong, Chengqing  and
      Xia, Fei  and
      Li, Wenjie  and
      Navigli, Roberto",
    booktitle = "Proceedings of the 59th Annual Meeting of the Association for Computational Linguistics and the 11th International Joint Conference on Natural Language Processing (Volume 1: Long Papers)",
    month = aug,
    year = "2021",
    address = "Online",
    publisher = "Association for Computational Linguistics",
    url = "https://aclanthology.org/2021.acl-long.250/",
    doi = "10.18653/v1/2021.acl-long.250",
    pages = "3226--3240",
}

@article{kurmanji2023towards,
  title={Towards unbounded machine unlearning},
  author={Kurmanji, Meghdad and Triantafillou, Peter and Hayes, Jamie and Triantafillou, Eleni},
  journal={Advances in neural information processing systems},
  volume={36},
  pages={1957--1987},
  year={2023}
}

@article{heng2023selective,
  title={Selective amnesia: A continual learning approach to forgetting in deep generative models},
  author={Heng, Alvin and Soh, Harold},
  journal={Advances in Neural Information Processing Systems},
  volume={36},
  pages={17170--17194},
  year={2023}
}

@inproceedings{dang2021right,
  title={Right to be forgotten in the age of machine learning},
  author={Dang, Quang-Vinh},
  booktitle={Advances in Digital Science: ICADS 2021},
  pages={403--411},
  year={2021},
  organization={Springer}
}

@inproceedings{
yao2024large,
title={Large Language Model Unlearning},
author={Yuanshun Yao and Xiaojun Xu and Yang Liu},
booktitle={The Thirty-eighth Annual Conference on Neural Information Processing Systems},
year={2024},
url={https://openreview.net/forum?id=8Dy42ThoNe}
}

@article{tarun2023fast,
  title={Fast yet effective machine unlearning},
  author={Tarun, Ayush K and Chundawat, Vikram S and Mandal, Murari and Kankanhalli, Mohan},
  journal={IEEE Transactions on Neural Networks and Learning Systems},
  year={2023},
  publisher={IEEE}
}

@article{jia2023model,
  title={Model sparsity can simplify machine unlearning},
  author={Jia, Jinghan and Liu, Jiancheng and Ram, Parikshit and Yao, Yuguang and Liu, Gaowen and Liu, Yang and Sharma, Pranay and Liu, Sijia},
  journal={Advances in Neural Information Processing Systems},
  volume={36},
  pages={51584--51605},
  year={2023}
}

@inproceedings{Thudi2021OnTN,
  title={On the Necessity of Auditable Algorithmic Definitions for Machine Unlearning},
  author={Anvith Thudi and Hengrui Jia and Ilia Shumailov and Nicolas Papernot},
  booktitle={USENIX Security Symposium},
  year={2021},
  url={https://api.semanticscholar.org/CorpusID:239616091}
}

@inproceedings{
fan2024simplicity,
title={Simplicity Prevails: Rethinking Negative Preference Optimization for {LLM} Unlearning},
author={Chongyu Fan and Jiancheng Liu and Licong Lin and Jinghan Jia and Ruiqi Zhang and Song Mei and Sijia Liu},
booktitle={Neurips Safe Generative AI Workshop 2024},
year={2024},
url={https://openreview.net/forum?id=pVACX02m0p}
}

@inproceedings{
zou2024improving,
title={Improving Alignment and Robustness with Circuit Breakers},
author={Andy Zou and Long Phan and Justin Wang and Derek Duenas and Maxwell Lin and Maksym Andriushchenko and J Zico Kolter and Matt Fredrikson and Dan Hendrycks},
booktitle={The Thirty-eighth Annual Conference on Neural Information Processing Systems},
year={2024},
url={https://openreview.net/forum?id=IbIB8SBKFV}
}

@inproceedings{
li2024the,
title={The {WMDP} Benchmark: Measuring and Reducing Malicious Use with Unlearning},
author={Nathaniel Li and Alexander Pan and Anjali Gopal and Summer Yue and Daniel Berrios and Alice Gatti and Justin D. Li and Ann-Kathrin Dombrowski and Shashwat Goel and Gabriel Mukobi and Nathan Helm-Burger and Rassin Lababidi and Lennart Justen and Andrew Bo Liu and Michael Chen and Isabelle Barrass and Oliver Zhang and Xiaoyuan Zhu and Rishub Tamirisa and Bhrugu Bharathi and Ariel Herbert-Voss and Cort B Breuer and Andy Zou and Mantas Mazeika and Zifan Wang and Palash Oswal and Weiran Lin and Adam Alfred Hunt and Justin Tienken-Harder and Kevin Y. Shih and Kemper Talley and John Guan and Ian Steneker and David Campbell and Brad Jokubaitis and Steven Basart and Stephen Fitz and Ponnurangam Kumaraguru and Kallol Krishna Karmakar and Uday Tupakula and Vijay Varadharajan and Yan Shoshitaishvili and Jimmy Ba and Kevin M. Esvelt and Alexandr Wang and Dan Hendrycks},
booktitle={Forty-first International Conference on Machine Learning},
year={2024},
url={https://openreview.net/forum?id=xlr6AUDuJz}
}

@article{gandikota2024elm,
  title={Erasing Conceptual Knowledge from Language Models},
  author={Rohit Gandikota and Sheridan Feucht and Samuel Marks and David Bau},
  journal={arXiv preprint arXiv:2410.02760},
  year={2024}
}

@misc{eval-harness,
  author       = {Gao, Leo and Tow, Jonathan and Abbasi, Baber and Biderman, Stella and Black, Sid and DiPofi, Anthony and Foster, Charles and Golding, Laurence and Hsu, Jeffrey and Le Noac'h, Alain and Li, Haonan and McDonell, Kyle and Muennighoff, Niklas and Ociepa, Chris and Phang, Jason and Reynolds, Laria and Schoelkopf, Hailey and Skowron, Aviya and Sutawika, Lintang and Tang, Eric and Thite, Anish and Wang, Ben and Wang, Kevin and Zou, Andy},
  title        = {The Language Model Evaluation Harness},
  month        = 07,
  year         = 2024,
  publisher    = {Zenodo},
  version      = {v0.4.3},
  doi          = {10.5281/zenodo.12608602},
  url          = {https://zenodo.org/records/12608602}
}

@inproceedings{
deng2025guard,
title={{GUARD}: Generation-time {LLM} Unlearning via Adaptive Restriction and Detection},
author={Zhijie Deng and Chris Yuhao Liu and Zirui Pang and Xinlei He and Lei Feng and Qi Xuan and Zhaowei Zhu and Jiaheng Wei},
booktitle={ICML 2025 Workshop on Machine Unlearning for Generative AI},
year={2025},
url={https://openreview.net/forum?id=HuIocuFAkY}
}

@inproceedings{

wang2026dragon,
title={{DRAGON}: Guard {LLM} Unlearning in Context via Negative Detection and Reasoning},
author={Yaxuan Wang and Chris Yuhao Liu and Quan Liu and Jinlong Pang and Wei Wei and Yujia Bao and Yang Liu},
booktitle={The Fourteenth International Conference on Learning Representations},
year={2026},
url={https://openreview.net/forum?id=vQLUAkl5SG}
}

@inproceedings{
pawelczyk2024incontext,
title={In-Context Unlearning: Language Models as Few-Shot Unlearners},
author={Martin Pawelczyk and Seth Neel and Himabindu Lakkaraju},
booktitle={Forty-first International Conference on Machine Learning},
year={2024},
url={https://openreview.net/forum?id=GKcwle8XC9}
}

@article{Jeung2025DUSKDN,
  title={DUSK: Do Not Unlearn Shared Knowledge},
  author={Wonje Jeung and Sangyeon Yoon and Hyesoo Hong and Soeun Kim and Seungju Han and Youngjae Yu and Albert No},
  journal={ArXiv},
  year={2025},
  volume={abs/2505.15209},
  url={https://api.semanticscholar.org/CorpusID:278782469}
}

@inproceedings{reisizadeh-etal-2026-blur,
    title = "{BLUR}: A Bi-Level Optimization Approach for {LLM} Unlearning",
    author = "Reisizadeh, Hadi  and
      Jia, Jinghan  and
      Bu, Zhiqi  and
      Vinzamuri, Bhanukiran  and
      Ramakrishna, Anil  and
      Chang, Kai-Wei  and
      Cevher, Volkan  and
      Liu, Sijia  and
      Hong, Mingyi",
    editor = "Demberg, Vera  and
      Inui, Kentaro  and
      Marquez, Llu{\'i}s",
    booktitle = "Proceedings of the 19th Conference of the {E}uropean Chapter of the {A}ssociation for {C}omputational {L}inguistics (Volume 1: Long Papers)",
    month = mar,
    year = "2026",
    address = "Rabat, Morocco",
    publisher = "Association for Computational Linguistics",
    url = "https://aclanthology.org/2026.eacl-long.331/",
    doi = "10.18653/v1/2026.eacl-long.331",
    pages = "7043--7058",
    ISBN = "979-8-89176-380-7",
}

@inproceedings{jeung-etal-2025-seps,
    title = "{SEPS}: A Separability Measure for Robust Unlearning in {LLM}s",
    author = "Jeung, Wonje  and
      Yoon, Sangyeon  and
      No, Albert",
    editor = "Christodoulopoulos, Christos  and
      Chakraborty, Tanmoy  and
      Rose, Carolyn  and
      Peng, Violet",
    booktitle = "Proceedings of the 2025 Conference on Empirical Methods in Natural Language Processing",
    month = nov,
    year = "2025",
    address = "Suzhou, China",
    publisher = "Association for Computational Linguistics",
    url = "https://aclanthology.org/2025.emnlp-main.283/",
    doi = "10.18653/v1/2025.emnlp-main.283",
    pages = "5556--5587",
    ISBN = "979-8-89176-332-6",
}

@misc{wang2023trace,
      title={TRACE: A Comprehensive Benchmark for Continual Learning in Large Language Models}, 
      author={Xiao Wang and Yuansen Zhang and Tianze Chen and Songyang Gao and Senjie Jin and Xianjun Yang and Zhiheng Xi and Rui Zheng and Yicheng Zou and Tao Gui and Qi Zhang and Xuanjing Huang},
      year={2023},
      eprint={2310.06762},
      archivePrefix={arXiv},
      primaryClass={cs.CL},
      url={https://arxiv.org/abs/2310.06762}, 
}

@inproceedings{
lu2021codexglue,
title={Code{XGLUE}: A Machine Learning Benchmark Dataset for Code Understanding and Generation},
author={Shuai Lu and Daya Guo and Shuo Ren and Junjie Huang and Alexey Svyatkovskiy and Ambrosio Blanco and Colin Clement and Dawn Drain and Daxin Jiang and Duyu Tang and Ge Li and Lidong Zhou and Linjun Shou and Long Zhou and Michele Tufano and MING GONG and Ming Zhou and Nan Duan and Neel Sundaresan and Shao Kun Deng and Shengyu Fu and Shujie LIU},
booktitle={Thirty-fifth Conference on Neural Information Processing Systems Datasets and Benchmarks Track (Round 1)},
year={2021},
url={https://openreview.net/forum?id=6lE4dQXaUcb}
}

@inproceedings{
lu2022learn,
title={Learn to Explain: Multimodal Reasoning via Thought Chains for Science Question Answering},
author={Pan Lu and Swaroop Mishra and Tony Xia and Liang Qiu and Kai-Wei Chang and Song-Chun Zhu and Oyvind Tafjord and Peter Clark and Ashwin Kalyan},
booktitle={Advances in Neural Information Processing Systems},
editor={Alice H. Oh and Alekh Agarwal and Danielle Belgrave and Kyunghyun Cho},
year={2022},
url={https://openreview.net/forum?id=HjwK-Tc_Bc}
}

@inproceedings{mishra-etal-2022-numglue,
    title = "{N}um{GLUE}: A Suite of Fundamental yet Challenging Mathematical Reasoning Tasks",
    author = "Mishra, Swaroop  and
      Mitra, Arindam  and
      Varshney, Neeraj  and
      Sachdeva, Bhavdeep  and
      Clark, Peter  and
      Baral, Chitta  and
      Kalyan, Ashwin",
    editor = "Muresan, Smaranda  and
      Nakov, Preslav  and
      Villavicencio, Aline",
    booktitle = "Proceedings of the 60th Annual Meeting of the Association for Computational Linguistics (Volume 1: Long Papers)",
    month = may,
    year = "2022",
    address = "Dublin, Ireland",
    publisher = "Association for Computational Linguistics",
    url = "https://aclanthology.org/2022.acl-long.246/",
    doi = "10.18653/v1/2022.acl-long.246",
    pages = "3505--3523",
}

@InProceedings{pmlr-v119-karimireddy20a,
  title = 	 {{SCAFFOLD}: Stochastic Controlled Averaging for Federated Learning},
  author =       {Karimireddy, Sai Praneeth and Kale, Satyen and Mohri, Mehryar and Reddi, Sashank and Stich, Sebastian and Suresh, Ananda Theertha},
  booktitle = 	 {Proceedings of the 37th International Conference on Machine Learning},
  pages = 	 {5132--5143},
  year = 	 {2020},
  editor = 	 {III, Hal Daumé and Singh, Aarti},
  volume = 	 {119},
  series = 	 {Proceedings of Machine Learning Research},
  month = 	 {13--18 Jul},
  publisher =    {PMLR},
  url = 	 {https://proceedings.mlr.press/v119/karimireddy20a.html},
}

@inproceedings{fedcm,
author = {Wang, Peng and Lu, Shoupeng and Yin, Hao and Yang, Banglie and Zhu, Tianli and Dai, Cheng},
title = {FedCM: client clustering and migration in federated learning via gradient path similarity and update direction deviation},
year = {2025},
isbn = {978-1-956792-06-5},
url = {https://doi.org/10.24963/ijcai.2025/706},
doi = {10.24963/ijcai.2025/706},
booktitle = {Proceedings of the Thirty-Fourth International Joint Conference on Artificial Intelligence},
articleno = {706},
numpages = {9},
location = {Montreal, Canada},
series = {IJCAI '25}
}

@misc{karimireddy2021mime,
      title={Mime: Mimicking Centralized Stochastic Algorithms in Federated Learning}, 
      author={Sai Praneeth Karimireddy and Martin Jaggi and Satyen Kale and Mehryar Mohri and Sashank J. Reddi and Sebastian U. Stich and Ananda Theertha Suresh},
      year={2021},
      eprint={2008.03606},
      archivePrefix={arXiv},
      primaryClass={cs.LG},
      url={https://arxiv.org/abs/2008.03606}, 
}

@inproceedings{local_adam,
 author = {Cheng, Ziheng and Glasgow, Margalit},
 booktitle = {International Conference on Learning Representations},
 editor = {Y. Yue and A. Garg and N. Peng and F. Sha and R. Yu},
 pages = {57694--57751},
 title = {Convergence of Distributed Adaptive Optimization with Local Updates},
 volume = {2025},
 year = {2025}
}

@misc{konecny2015federated,
      title={Federated Optimization:Distributed Optimization Beyond the Datacenter}, 
      author={Jakub Konečný and Brendan McMahan and Daniel Ramage},
      year={2015},
      eprint={1511.03575},
      archivePrefix={arXiv},
      primaryClass={cs.LG},
      url={https://arxiv.org/abs/1511.03575}, 
}
